%% file: THV_v5_AAAI.tex
\documentclass[letterpaper]{article} 
\usepackage[preprint]{aaai2027}  
\usepackage[hyphens]{url}  
\usepackage{graphicx} 
\usepackage{natbib}  
\usepackage{caption} 
\usepackage{algorithm}
\usepackage{algorithmic}
\usepackage{cuted}
\usepackage{newfloat}
\usepackage{listings}
\DeclareCaptionStyle{ruled}{labelfont=normalfont,labelsep=colon,strut=off} 
\floatstyle{ruled}
\newfloat{listing}{tb}{lst}{}
\floatname{listing}{Listing}

\usepackage{booktabs}

\usepackage{amsmath}
\usepackage{amsthm}
\usepackage{amsfonts}
\usepackage{amssymb}
\usepackage{comment}
\usepackage{subcaption}

\newtheorem*{remark}{Remark}
\newtheorem{theorem}{Theorem}
\newtheorem{lemma}{Lemma}
\newtheorem{corollary}{Corollary}

\newtheorem{fact}{Fact}

\newcommand{\ttn}[1]{\textcolor{black}{#1}}

\newcommand{\faithful}{$\bullet$}    
\newcommand{\extension}{$\circ$}     

\title{Top-$k$ Pareto Bandits: Hypervolume Regret for Multi-Objective Slate Selection}

\author{
    Nicolas Gutowski\textsuperscript{\rm 1},   Fabien Chhel\textsuperscript{\rm 2,1}, Alexandre Letard\textsuperscript{\rm 4,3,1},  Sylvain Lamprier\textsuperscript{\rm 1}\\
}
\affiliations{
    \textsuperscript{\rm 1}Université d'Angers, LERIA, SFR MATHSTIC,  Angers, France\\
    \textsuperscript{\rm 2}ESEO, ERIS, Angers, France\\
    \textsuperscript{\rm 3}ESAIP, Saint-Barthélémy-d'Anjou, France\\
    \textsuperscript{\rm 4}AlphaEdge, Saint-Barthélémy-d'Anjou, France\\

    \{nicolas.gutowski, sylvain.lamprier\}@univ-angers.fr \\
    fabien.chhel@eseo.fr, alexandre.letard@alphaedge-ai.com 
}

\begin{document}

\maketitle

\begin{abstract}
We consider a stochastic multi-objective bandit problem where, at each round, the agent selects a slate of $k$ arms and observes their $d$-dimensional reward vectors under semi-bandit feedback. We do not aim at identifying a single optimal arm; instead, we consider the problem of maintaining a small set of actions that jointly approximate the Pareto frontier. We formalize this objective through the dominated hypervolume induced by the selected subset of arms, and define an $\alpha$-approximate hypervolume regret with respect to the best size-$k$
subset achievable in hindsight, where $\alpha = 1 - 1/e$ reflects the approximation guarantee of greedy maximization for monotone submodular functions. To address this problem, we introduce \textit{THV-UCB}, an
optimistic algorithm that selects arms greedily based on optimistic estimates of their marginal hypervolume contributions. We establish a gap-free regret bound $\tilde{O}(d\sqrt{nkT})$ that holds on every instance, together with a gap-dependent bound $\tilde{O}(nk^{2.5}/\Delta_{\min})$ that becomes polylogarithmic in $T$ once the arms are sufficiently well separated.
Our results provide theoretical support for using small subsets to approximate Pareto fronts in various multi-objective applications.
\end{abstract}

\section{Introduction}
\label{sec:intro}

Most real-world decision problems involve balancing several conflicting criteria, and the corresponding paradigm of Multi-Objective Optimization (MOO) seeks not a single optimum but a Pareto-optimal set of trade-off solutions~\cite{Tian_2022_ACM,Fromer_2023_Patterns}. This challenge extends to the multi-armed bandit (MAB) framework where each pull yields not a scalar but a vector of rewards that captures competing criteria: accuracy vs.\ diversity in recommender systems~\cite{Letard_2024_ESWA,Ezzahra_2025_IJMIR}, profit vs.\ inventory in market making~\cite{Vicente_2026_ESWA}, efficacy vs.\ toxicity in clinical trials~\cite{Kone2023RelaxedPareto, Kone_2025b_PMLR}.

Three complementary aims drive contemporary research: 1)~approximating the Pareto front faithfully and uniformly~\cite{Kone2023RelaxedPareto,Kone2025aConstrainedPareto,Kone_2025b_PMLR,Shahverdikondori2025Group}, 2)~doing so under the noisy, sample-limited feedback that characterizes online and bandit settings to minimize Pareto regret~\cite{Mandow2023ChebyshevMAB, xu2023pareto, Cao2025PreferenceBandit, Huyuk_2021_ML,Xue_2025_AAAI}, and 3)~measuring progress through unary, preference-free indicators such as the dominated hypervolume~\cite{Guerreiro_2021_ACM}.

This third line formalizes MOO progress through the dominated hypervolume (HV)~\cite{Guerreiro_2021_ACM}, the only preference-free indicator that is strictly Pareto-compliant. Hypervolume maximization emerges as a more natural quality criterion in scenarios where the coverage of the Pareto front matters most. HV has been frequently used as a training signal in Pareto Set Learning~\cite{Zhang_2023_NeurIPS, Zhang2023HypervolumeRegret} and in multi-objective reinforcement learning~\cite{Liu_2025_AAAI, Liu_2025_ICLR, Lee_2026_ESWA, Roepke_2025_AAMAS, Song_2026_IEEE, Vicente_2026_ESWA, Letard_2024_ESWA}.

Previous works, however, either operate in continuous black-box domains, query a single point per  round, do not scale with the number of objective $d > 2$, lack theoretical foundations or focus on a specific kind of problem (e.g. concave or convex), lacking applicability and robustness. 




To the best of our knowledge, no previous work has addressed these challenges in a cross-domain online setting. We bridge the three main lines of research in the multi-objective optimization field by introducing \textit{THV-UCB}. This algorithm leverages optimistic reward vectors to greedily maximize marginal hypervolume gains, while utilizing coordinate-wise confidence boxes to safely prune dominated arms and perform initial forced exploration.

Hence, we formalize a stochastic multi-objective bandit problem where, at each round $t$, the agent selects a slate $S_t$ of $k$ arms from a set of $n$ candidates and observes their $d-$dimensional reward vectors under semi-bandit feedback. The performance of $S_t$ is evaluated by the dominated hypervolume it covers relative to a reference point, with our theoretical analysis comparing this performance against the optimal hypervolume achievable by any subset of size $k$.

Our main contributions can be summarized as follows :

\begin{enumerate}
    \item We introduce the Top-$k$ Pareto bandit setting and define an $\alpha$-approximation hypervolume regret with respect to the best size-$k$ subset of the Pareto frontier, more suited for many real-world scenarios ;
    \item We extend previous competitive works \cite{Drugan2013MOBandits, Deb2002NSGA2, yahyaa2015thompson, Mandow2023ChebyshevMAB, auer2002finite, paria2020flexible, zhang2020random, Zhang2023HypervolumeRegret} from Pareto optimization and scalarization methods to this setting and empirically evaluate them for hypervolume maximization in top-$k$ semi-bandit setting considering 1)~four synthetic fronts (linear, convex, concave and clusters) ; 2)~$d \in \mathbb{Z} \cap [2,5] $ conflicting objectives (dimensions of the Pareto front) and associated top-$k \in \mathbb{Z} \cap [3,6] $ (slates-length - arms to be selected at each round) ; 
    \item We propose \textsc{THV-UCB}, an optimistic algorithm using coordinate-wise $\ell_\infty$ confidence boxes and greedy selection on optimistic marginal HV gains. The construction differs from the random HV scalarizations of~\cite{zhang2020random,Zhang_2024_NeurIPS} by directly exploiting the submodularity of HV in a discrete $k$-armed slate setting. Empirically, \textit{THV-UCB} achieves the \emph{lowest cumulative $\alpha$-regret} and the \emph{highest hypervolume} in all four front geometries and dimensions by a margin increasing with $d$. Theoretically, we prove a gap-free  regret bound~$\tilde{O}\!\left(d\sqrt{nkT}\right)$, and a gap-dependent  regret bound~$\tilde{O}\!\left(\frac{nk^{2.5}}{\Delta_{\min}}\right)$ that are polylogarithmic in $T$. 
\end{enumerate}

The paper is organized as follows: Related Work reviews prior literature. Problem Setting depicts the problem setting and regret definition, while Top-$k$ HyperVolume UCB presents the proposed algorithm THV-UCB. Regret Analysis exposes our theoretical analysis of the method and establishes upper bounds on the regret. Finally, Experiments describes our experimental evaluation.

\section{Related Work}
\label{sec:related}

\paragraph{Multi-Objective Optimization (MOO).} 
MOO has a long history of study, with a continued stream of recent work
\cite{Tian_2022_ACM, Ghanbarzadeh_2026_AIR, Jiju_2025_IJAM, Yifan_2025_PCS, He_2026_ACM, Ezzahra_2025_IJMIR}, extending many research fields. Comparable recent works in deep reinforcement learning made use of hypervolume both as a quality criterion and an optimization mean to reduce the computation cost of learning the whole Pareto Front \cite{Zhang_2023_NeurIPS, CAI_2023_NeurIPS, Jinbiao_2023_NeurIPS, Lee_2026_ESWA, Vicente_2026_ESWA}. More specifically, HV-driven Pareto Set Learning methods maximizes HV via gradient descent on a neural preference-conditioned model~\cite{Zhang_2023_NeurIPS, Zhang2023HypervolumeRegret}, while HV-based MORL embeds HV in policy optimization~\cite{Roepke_2025_AAMAS, Song_2026_IEEE}. However, as stated by Zhang et al. \cite{Zhang_2024_NeurIPS}, a major drawback of gradient-based methods for hypervolume maximization is the high computational complexity in obtaining the hypervolume gradient. While showing good results up to $d=4$ objectives, these approaches remain unpractical for an online setup.

\paragraph{Multi-Objective Multi-Armed Bandits (MOMAB).} Multi-objective multi-armed bandits extend the classical bandit framework to vector-valued rewards and Pareto-based notions of optimality. Early work introduced the setting and adapted UCB/TS principles to the multi-objective scenario considering Pareto regret~\cite{Drugan2013MOBandits, Yahyaa2014KG, yahyaa2015thompson, Roijers2017InteractiveTS, xu2023pareto}.
Following them, several works extend regret-minimizing MOMAB through various scalarization techniques like : Chebyshev~\cite{Mandow2023ChebyshevMAB}, lexicographic priorities~\cite{Huyuk_2021_ML}, lexicographic linear bandits~\cite{Xue_2025_AAAI}, and preference-aware customization~\cite{Cao2025PreferenceBandit}.
A parallel line studies \emph{pure exploration} goals, such as identifying feasible arms or approximating the Pareto set up to a relaxation tolerance, including $\varepsilon$-relaxed Pareto set identification, constrained variants, and best-group  Identification~\cite{KatzSamuels2018Feasible, Kone2023RelaxedPareto, Kone2025aConstrainedPareto, Shahverdikondori2025Group}. These works typically aim to recover a large fraction of the Pareto set rather than maintaining a small representative \emph{subset} of fixed size $k$.
Scalarization methods, considering either a single aggregated utility function \cite{BusaFekete2017Gini, roijers2013survey, Roijers2017InteractiveTS,  Mandow2023ChebyshevMAB} or several different trade-offs \cite{zhang2020random, Letard_2024_ESWA, Zhang2023HypervolumeRegret, Cao2025PreferenceBandit, Liu_2025_AAAI, Liu_2025_ICLR}, are often most convenient for the online bandit setup. However these methods involve a fixed distribution of preferences~\cite{roijers2013survey,Roijers2017InteractiveTS, Zhang2023HypervolumeRegret} and do not directly address \emph{set-level} hypervolume of a \emph{size-$k$ slate} under semi-bandit feedback which is complementary to our goal of \emph{preference-free} coverage of the frontier.

\paragraph{Dominated Hypervolume (HV) in MOMAB.} To the best of our knowledge, little work made use of the dominated hypervolume as a learning target in MOMAB problems. In black-box multi-objective optimization, random hypervolume scalarizations provide provable guarantees for exploring Pareto trade-offs~\cite{zhang2020random}  with $\tilde{O}(\sqrt T)$ HV-regret bounds for UCB/TS Bayesian optimization. Later, these bounds were further refined by~\cite{Zhang2023HypervolumeRegret}, who establish an optimal hypervolume regret bound  of $O(T^{-1/k})$ 
The problems studied by Kone et al. \cite{Kone2023RelaxedPareto, Kone2025aConstrainedPareto} and~\cite{Zhang2023HypervolumeRegret} are the closest to ours. Nevertheless, despite the added diversity in arm selection decomposition methods are mainly relevant for Pareto Front Identification, with lower performance for hypervolume maximization (see the Experiments section).
In terms of concept, the closest existing methods to our proposed THV-UCB algorithm are the HV-based approaches proposed by \cite{Zhang_2023_NeurIPS} and \cite{Zhang_2024_NeurIPS} from MORL literature. However, these works do not adress the discrete $k$-armed slate bandit setting we study.

\section{Problem Setting}
\label{sec:model}

\paragraph{Arms, horizon, and rewards.}
We consider $n$ arms indexed by $[n]=\{1,\dots,n\}$ over a horizon of $T$ rounds, with $[T]=\{1,\dots,T\}$.
Pulling arm $i$ at round $t$ yields a $d$-dimensional random reward vector $\mathbf{X}_{i,t}\in[0,1]^d$ with unknown mean
$\boldsymbol{\mu}_i = \mathbb{E}[\mathbf{X}_{i,t}]$.
We assume $(\mathbf{X}_{i,t})_{t\in[T]}$ are independent across $t$ and $i$, and
each coordinate is $\eta$-sub-Gaussian (used for concentration); $\eta$
also serves as the confidence parameter of the algorithm.

\paragraph{Top-$k$ actions and semi-bandit feedback.}
At each round $t\in[T]$, the agent selects a subset $S_t\subseteq[n]$ of size $|S_t|=k$ (a \emph{slate}) and observes the vectors $\{\mathbf{X}_{i,t}: i\in S_t\}$ (semi-bandit feedback).

\paragraph{Dominance and Pareto front.}
For $\mathbf{a},\mathbf{b}\in\mathbb{R}^d$, write $\mathbf{a}\preceq \mathbf{b}$ if $a_j\le b_j$ for all $j\in[d]$, and $\mathbf{a}\prec \mathbf{b}$ if in addition at least one inequality is strict.
Arm $i$ \emph{Pareto-dominates} arm $j$ if $\boldsymbol{\mu}_i \succeq \boldsymbol{\mu}_j$ and $\boldsymbol{\mu}_i\neq \boldsymbol{\mu}_j$.
The Pareto front $\mathcal{P}$ is the set of undominated mean vectors $\{\boldsymbol{\mu}_i:i\in[n]\}$.

\paragraph{Reference point and dominated hypervolume.}
Fix a reference point $\mathbf{r}\in\mathbb{R}^d$ such that $\mathbf{r}\preceq \boldsymbol{\mu}_i$ for all $i\in[n]$ (e.g., $\mathbf{r}=\mathbf{0}$ when rewards lie in $[0,1]^d$).
For any subset $S\subseteq[n]$, its dominated hypervolume is
\begin{equation}
\label{eq:hv-def}
\mathrm{HV}(S)
=
\lambda_d\!\left(
\bigcup_{i\in S}
[r_1,\mu_{i1}] \times \cdots \times [r_d,\mu_{id}]
\right)
\end{equation}
where $\lambda_d$ is the $d$-dimensional Lebesgue measure. \ttn{The dominated hypervolume is a standard performance indicator in multi-objective optimization~\cite{Zitzler2003PerformanceAssessment}.}\footnote{%
Any equivalent definition of dominated hypervolume can be used.
Our analysis only relies on monotonicity and submodularity of $\mathrm{HV}(\cdot)$ as a set function over mean vectors.}


\paragraph{Performance metric: (approximation) hypervolume regret.}
We evaluate performance in terms of pseudo-regret with respect to the
mean rewards~\cite{Bubeck2012RegretAO}. Let
$S^\star \in \arg\max_{|S|=k} \mathrm{HV}(S)$ denote a best subset of
size $k$ in hindsight, and $V^\star = \mathrm{HV}(S^\star)$ its
hypervolume. The (ideal) instantaneous regret and cumulative regret are
\[
r_t = V^\star - \mathrm{HV}(S_t), \quad
R_T = \sum_{t=1}^{T} r_t.
\label{eq:regret}
\]
Since maximizing a monotone submodular function under a cardinality
constraint is NP-hard and typically addressed by greedy selection, which
achieves approximation factor $\alpha = 1 - 1/e$, our guarantees are
stated for the $\alpha$-approximation regret
\begin{equation}
\bar{r}_t = \alpha V^\star - \mathrm{HV}(S_t), \quad
\bar{R}_T = \sum_{t=1}^{T} \bar{r}_t.
\label{eq:regret-alpha}
\end{equation}
This benchmark cleanly separates computational approximation (the
$\alpha$ factor, unavoidable for any polynomial-time algorithm) from
statistical learning (the gap between $\mathrm{HV}(S_t)$ and what the
algorithm could achieve with known means).

\paragraph{Additional notation.}
We write $\log$ for the natural logarithm and $\tilde{O}(\cdot)$ to hide polylogarithmic factors.
Vectors are bold lowercase, sets uppercase, and $\|\cdot\|_\infty$ denotes the $\ell_\infty$ norm.
For any arm $i$, let $N_i(t)$ be the number of times $i$ has been selected up to (and including) round $t$, and let $\widehat{\boldsymbol{\mu}}_i(t)$ denote its empirical mean vector.
Table~\ref{tab:notation} in the appendix.

\section{\textit{Top-$k$} HyperVolume UCB (\textit{THV-UCB})}
\label{sec:alg}

We now present \textit{THV-UCB}, an optimistic algorithm that maintains coordinate-wise UCB boxes for each arm's mean vector and constructs at every round a size-$k$ subset by greedily maximizing the \emph{optimistic} marginal hypervolume gain.

\paragraph{UCB boxes (coordinate-wise optimism).}
At round $t$, for every arm $i$ and objective $j$, we form an upper confidence bound 
\[
u_{i,j}(t) = \widehat{\mu}_{i,j}(t-1) \;+\; \beta_i(t)
\]
with
\[
\beta_i(t) = \sqrt{\frac{2\eta \log\!( nd\, t^2/\delta)}{\max\{1,N_i(t-1)\}}}
\]
where $\eta>0$ is a confidence parameter and $\delta\in(0,1)$ is the target failure probability.
This yields an optimistic vector $\mathbf{u}_i(t) = (u_{i,1}(t),\dots,u_{i,d}(t))$.
Optionally, since rewards lie in $[0,1]^d$, we clip $\mathbf{u}_i(t)$ coordinate-wise to $[0,1]$.

\paragraph{THV-UCB: Greedy Subset Construction.} 

Algorithm~\ref{alg:thvucb} describes our proposed method \textit{THV-UCB}\footnote{The code of \textit{THV-UCB} is available in our GitHub repository \url{https://github.com/ngutowski/topk-pareto-bandits}}, which aims to maximize the hypervolume defined in~\eqref{eq:hv-def} via a greedy construction of the subset $S_t$ based on optimistic estimates. After an initialization phase, the algorithm computes UCB-based confidence intervals for each arm and applies a safe pruning step to form a candidate set $A_t$ of arms that are not confidently dominated, i.e., arms that may still contribute to an optimal solution.

The optimistic hypervolume $HV_t^{\mathrm{UCB}}(S)$ is obtained by replacing the unknown means $\mu_{i,j}$ in~\eqref{eq:hv-def} with their UCB counterparts $u_{i,j}(t)$. For any set $S \subseteq [n]$ and arm $i \notin S$, we define the marginal UCB hypervolume gain as
$
\Delta^{\mathrm{UCB}}_{\mathrm{HV},t}(i \mid S)
= HV_t^{\mathrm{UCB}}(S \cup \{i\}) - HV_t^{\mathrm{UCB}}(S).
$

The subset $S_t$ is then constructed greedily: starting from $S=\emptyset$, arms are sequentially added by maximizing $\Delta^{\mathrm{UCB}}_{\mathrm{HV},t}(i \mid S)$. This procedure yields a $(1 - 1/e)$-approximation of $
\max_{|S|=k} HV_t^{\mathrm{UCB}}(S),
$ 
by standard results on monotone submodular maximization. The selected arms in $S_t$ are then pulled and the statistics updated.

From a computational perspective, computing $\Delta^{\mathrm{UCB}}_{\mathrm{HV},t}(i \mid S)$ for all candidates at each greedy step leads to a per-round complexity of $O (k n \cdot \mathrm{costHV})$, where $\mathrm{costHV}$ denotes the cost of updating the hypervolume. In practice, incremental updates, dominance pruning, and the regime $k \ll n$ make the greedy selection efficient.

\begin{algorithm}[ht!]
\caption{\textit{THV-UCB}: Top-$k$ HyperVolume UCB}
\label{alg:thvucb}
\begin{algorithmic}[1]
\STATE \textbf{Input:} subset size $k$, reference point $\mathbf{r}$, confidence parameter $\eta$, failure level $\delta$, horizon $T$, minimum pulls $m_0 \ge 1$ (default: $m_0=2$)
\STATE \textbf{Init:} $N_i(0) \gets 0$, $\widehat{\boldsymbol{\mu}}_i(0) \gets \mathbf{0}$ for all $i \in [n]$; set $t \gets 1$
\WHILE{$\exists\, i \in [n]$ s.t. $N_i(t-1) < m_0$}
    \STATE $S_t \gets$ the $k$ arms with smallest $N_i(t-1)$ (ties broken arbitrarily)
    \STATE Play all arms in $S_t$, observe $\{\mathbf{X}_{i,t}: i\in S_t\}$, update $N_i(t)$, $\widehat{\boldsymbol{\mu}}_i(t)$
    \STATE $t \gets t+1$
\ENDWHILE
\FOR{$t$ to $T$}
    \FOR{each arm $i\in[n]$}
        
        \STATE  $\beta_i(t) \gets \sqrt{\dfrac{2\eta \log(nd\,t^2/\delta)}{\max\{1,N_i(t-1)\}}}$
        
        \STATE $\mathbf{u}_i(t) \gets \widehat{\boldsymbol{\mu}}_i(t-1) + \beta_i(t)\,\mathbf{1}_d$
        \STATE $\mathbf{L}_i(t) \gets \widehat{\boldsymbol{\mu}}_i(t-1) - \beta_i(t)\,\mathbf{1}_d$
    \ENDFOR
    \STATE \textit{Safe pruning:}
    \STATE \hspace{2.8em} $A_t \leftarrow \{ i \in [n] : \forall\, j \neq i,\ \mathbf{u}_i(t) \npreceq \mathbf{L}_j(t) \}$
    
    \STATE \hspace{2.8em} if $|A_t| < k$, set $A_t \gets [n]$
    \STATE $S \gets \emptyset$
    \WHILE{$|S|<k$}
        \STATE Select $i^\star \in \arg\max_{i\in A_t\setminus S} \Delta^{\mathrm{UCB}}_{\mathrm{HV},t}(i\mid S)$
        \STATE $S \gets S \cup \{i^\star\}$
    \ENDWHILE
    \STATE Play all arms in $S$, observe $\{\mathbf{X}_{i,t}: i\in S\}$, update $N_i(t)$, $\widehat{\boldsymbol{\mu}}_i(t)$ for $i\in S$
\ENDFOR
\end{algorithmic}
\end{algorithm}

\section{Regret Analysis}\label{sec:theoretical_analysis}

We establish two complementary guarantees on the $\alpha$-approximation hypervolume regret of THV-UCB: a \emph{gap-free} bound, valid on every instance regardless of how close the arms are to each other, and a \emph{gap-dependent} bound, which becomes polylogarithmic in $T$ as soon as the instance is well separated. Since maximizing hypervolume under a cardinality constraint is NP-hard, THV-UCB relies on greedy maximization of a monotone submodular optimistic objective, and both guarantees are stated  for the $\alpha$-regret $\bar{R}_T$ with $\alpha = 1 - 1/e$ .

\paragraph{Benchmark and separation quantity.}
Recall $S^\star \in \arg\max_{|S|=k}\mathrm{HV}(S)$ and $V^\star = \mathrm{HV}(S^\star)$ defined in (2). For the gap-dependent analysis, we introduce a second, purely proof-internal benchmark: let $G^\star = \{i^\star_1,\dots,i^\star_k\}$ be built by greedy maximization of the \emph{true} marginal gains $\Delta(i \mid S) = \mathrm{HV}(S\cup\{i\}) - \mathrm{HV}(S)$, i.e.\ $i^\star_\ell \in \arg\max_{i \notin G^\star_{\ell-1}} \Delta(i \mid G^\star_{\ell-1})$, with $G^\star_\ell = G^\star_{\ell-1} \cup \{i^\star_\ell\}$ and $G^\star_0 = \emptyset$. Assuming the greedy maximizer is unique at every stage, define the stage-$\ell$ gaps and the minimum gap $\Delta_\ell(i) = \Delta(i^\star_\ell \mid G^\star_{\ell-1}) - \Delta(i \mid G^\star_{\ell-1}),\Delta_{\min} = \min_{\ell \in [k]} \min_{i \neq i^\star_\ell} \Delta_\ell(i)$.

Note that $G^\star$ never appears in the regret definition: the regret is always measured against $V^\star$; $G^\star$ only serves to track the algorithm's stage-wise greedy progress in the analysis.

We write $\mathcal{E}$ for the event on which all coordinate-wise confidence intervals are valid simultaneously; by a standard sub-Gaussian concentration argument and a union bound (Lemma~\ref{lem:conc}, see the appendix)
, $\mathbb{P}(\mathcal{E}) \geq 1 - \delta$. Finally, $C_d$ denotes the coordinate-wise Lipschitz constant of the hypervolume; one may take $C_d \leq d$ when rewards lie in $[0,1]^d$ and $r = 0$. (Lemma~\ref{lem:lip}, see the appendix).

\begin{theorem}[Gap-free bound (short version)]\label{thm:gapfree}
Assume each reward coordinate is $\eta$-sub-Gaussian and bounded in $[0,1]$. On $\mathcal{E}$,
\[
\bar{R}_T \;=\; O\!\left( C_d \sqrt{n k T \log T} \right).
\]
\end{theorem}

\begin{theorem}[Gap-dependent bound (short version)]\label{thm:gapdep} Under the same assumptions, if the greedy maximizer of $G^\star$ is unique at every stage and the safe-pruning step never eliminates an arm of $G^\star$, then on $\mathcal{E}$,
\[
\bar{R}_T \;=\; O\!\left( \frac{n k^{2.5} \log T}{\Delta_{\min}} \right).
\]
\end{theorem}

\begin{theorem}[Regret of THV-UCB (short version)]\label{thm:combined}
With $\delta = 1/T$,
\[
\begin{aligned}
    \mathbb{E}\big[\bar{R}_T\big] \leq {} \\ \min\Bigl\{ O\left(\sqrt{n k T \log T}\right),
    & O\left(\frac{n k^{2.5} \log T}{\Delta_{\min}}\right) \Bigr\} + O(1).
\end{aligned}
\]

In particular $\mathbb{E}[\bar{R}_T]/T \to 0$: THV-UCB achieves sublinear $\alpha$-approximation regret on every instance.
\end{theorem}

\paragraph{Proof sketch.}
Both bounds share the same reduction, then diverge.

\emph{(1) Optimism reduction.} On $\mathcal{E}$, coordinate-wise optimism gives $\mathrm{HV}(S) \leq \mathrm{HV}^{\mathrm{UCB}}_t(S)$ for every $S$. Since $\mathrm{HV}^{\mathrm{UCB}}_t$ is itself monotone submodular,
the greedy construction of $S_t$ is a $(1-1/e)$-approximation of its maximizer, hence $\mathrm{HV}^{\mathrm{UCB}}_t(S_t) \geq \alpha\,\mathrm{HV}^{\mathrm{UCB}}_t(S^\star)$, and 
$
\bar{r}_t \;\leq\; \mathrm{HV}^{\mathrm{UCB}}_t(S_t) - \mathrm{HV}(S_t).
$

\emph{(2) Lipschitz control.} The hypervolume is coordinate-wise Lipschitz, so the optimism error is at most
$C_d \sum_{i \in S_t} \beta_i(t)$, reducing the regret to a sum of confidence radii.

\emph{(3a) Gap-free control.} Reordering the double sum by arm and applying Cauchy--Schwarz over the whole horizon, using $\sum_i N_i(T) = kT$, yields Theorem~\ref{thm:gapfree}. This step is blind to the selection mechanism: it only uses $|S_t| = k$ and $\beta_i(t) = \Theta(1/\sqrt{N_i(t-1)})$.

\emph{(3b) Gap-dependent control.} We instead track, at each round, the first stage $\ell$ at which the algorithm's greedy chain departs from $G^\star$. On \emph{matched} rounds ($S_t = G^\star$) the $\alpha$-regret is non-positive. On \emph{deviation} rounds, a witness-counting argument shows that some arm in the current prefix must still have a large confidence radius, which caps the number of stage-$\ell$ deviations at $O(n \ell^2 \log T / \Delta_{\min}^2)$; applying the Cauchy--Schwarz argument of (3a) \emph{locally} to these rounds converts this $1/\Delta_{\min}^2$ count into a $1/\Delta_{\min}$ regret contribution. Summing over stages yields Theorem~\ref{thm:gapdep}.


A discussion and the detailed versions of Theorems~\ref{thm:gapfree}, \ref{thm:gapdep} and \ref{thm:combined} together with their full proofs, with all supporting lemmas and the corollary invoked above, are given in the appendix and in our GitHub repository \url{https://github.com/ngutowski/topk-pareto-bandits}.

\section{Experiments}
\label{sec:exp}

We empirically evaluate \textsc{THV-UCB} on controlled synthetic multi-objective bandit instances \ttn{across four Pareto front geometries and four dimensions ($d \in \{2, 3, 4, 5\}$, with slate size
$k \in \{3, 4, 5, 6\}$ and correspondingly varying total number of available arms $n$ and horizon $T$),} and compare it to representative baselines from the MO-bandit literature as well as scalarization-based methods \ttn{adapted to the top-$k$ semi-bandit setting.}

Experiments are conducted for $d \in \{2, 3, 4, 5\}$ objectives. The dominated hypervolume~\cite{Zitzler2003PerformanceAssessment} is computed exactly via inclusion-exclusion over the $k$-point slate, which remains tractable for small $k$.
Figure~\ref{fig:synthetic-fronts-linear}~to~\ref{fig:synthetic-fronts-convex} and Table~\ref{tab:synthetic-summary} report results for $d=2$; full results for $d \in \{3, 4, 5\}$, along with the grid search over the confidence parameter $\eta$ are provided in the appendix and in our GitHub repository \url{https://github.com/ngutowski/topk-pareto-bandits} (Tables~\ref{tab:eta-optimal},~\ref{tab:synthetic-summary-d3},~\ref{tab:synthetic-summary-d4}, and~\ref{tab:synthetic-summary-d5}, Figures~\ref{fig:synthetic-fronts-d3},~\ref{fig:synthetic-fronts-d4}, and~\ref{fig:synthetic-fronts-d5}). Moreover, note that all experiments were run on CPU only (Intel Xeon E5-2695 v4, 2.10~GHz, 45~MB cache), using Python~3.11.2, NumPy~1.24.2, and Matplotlib~3.6.3 for figure generation.

\subsection{Protocol and Metrics for $d=2$ and $k=3$}
\label{sec:exp:protocol}
\ttn{Unless stated otherwise, experiments use $d=2$, $n=36$, $k=3$, horizon $T=2000$,
Gaussian noise level $\sigma=0.05$, reference point $\mathbf{r}=\mathbf{0} \in \mathbb{R}^d$, and results are averaged over $10$ random seeds.}

\paragraph{Computing the benchmark $V^\star$.}
For $d=2$, $V^\star$ is computed by exhaustive enumeration over all
$\binom{n}{k}$ candidate subsets. For $d \geq 3$, exhaustive enumeration
becomes intractable, so we report $\mathrm{HV}(G^\star) \geq \alpha
V^\star$ (see Fact~\ref{thm:greedy-generic} in the appendix), a valid conservative proxy
that does not affect relative comparisons between methods. 

\ttn{At each round, we evaluate the selected slate $S_t$ using the pseudo-hypervolume
$\mathrm{HV}(\{\boldsymbol{\mu}_i : i\in S_t\})$ computed from the true means, to remove
observation noise from the metrics.
We report two complementary metrics:
(i)~the cumulative $\alpha$-regret
$\sum_{t=1}^T \bigl(\alpha V^\star - \mathrm{HV}(S_t)\bigr)$,
where $\alpha = 1 - 1/e$, which can be negative when a method consistently attains
hypervolume above $\alpha V^\star$;
and (ii)~the hypervolume trajectory $\mathrm{HV}(S_t)$, plotted as a moving average
(window $w=50$) with $95\%$ confidence intervals across seeds.}

\subsection{Synthetic Environments}
\label{sec:exp:env}

\ttn{We generate instances by mixing a structured Pareto frontier with dominated distractors.
A fraction $n_{\mathrm{front}} = \max\{10, \lfloor 0.35n \rfloor\}$ arms lie on a parametric
frontier defined via the angular parameterization of DTLZ~\cite{deb2005scalable}, while the remaining $n_{\mathrm{dom}} = n - n_{\mathrm{front}}$ arms are strictly dominated points sampled
uniformly in $[0, 0.3]^d$. Rewards are observed with additive Gaussian noise
$\mathcal{N}(0, \sigma^2)$ and clipped to $[0,1]^d$.}

\ttn{The frontier is parameterized by $d-1$ angles $\theta_j \sim \mathcal{U}[0, \pi/2]$,
with coordinates:}
\[
    \ttn{x_i = \left(\prod_{j=0}^{i-1} \sin^\alpha(\theta_j)\right) \cdot}
    \begin{cases}
        \ttn{\cos^\alpha(\theta_i)} & \ttn{\text{if } i < d-1,} \\
        \ttn{1    }                 & \ttn{\text{if } i = d-1,}
    \end{cases}
\]
\ttn{where the shape exponent $\alpha$ controls the front geometry.
We consider four geometries:
concave ($\alpha=1.0$, spherical front with $\sum_i x_i^2 = 1$, following DTLZ2~\cite{deb2005scalable});
convex ($\alpha=0.5$, outward-bulging front with $\sum_i x_i^4 = 1$);
linear ($\alpha=2.0$, simplex-like front);
and clusters ($\alpha=1.0$, spherical front with two disjoint angular regions,
$\theta_1 \in [0, \pi/5]$ for cluster~1 and $\theta_1 \in [3\pi/10, \pi/2]$ for cluster~2,
with remaining angles free in $[0, \pi/2]$).}

\ttn{For the linear geometry, the raw simplex coordinates collapse toward zero as $d$ grows,
making the dominated hypervolume uninformative. We therefore rescale the frontier points
by a factor $s = \min(0.45d,\, 2.5)$ and clip to $[0.01, 1]^d$, which preserves a
non-degenerate HV across dimensions while keeping coordinates in $[0,1]^d$.}
\subsection{Baselines}
\label{sec:exp:baselines}

\ttn{THV-UCB is compared against four learning-based families of baselines adapted to the top-$k$ semi-bandit setting, plus a non-learning baseline: 1) Pareto-layer UCB methods (\textsc{ParetoUCB}, \textsc{ParetoUCB-Div}, \textsc{ParetoUCB-Crowd}~\cite{Drugan2013MOBandits, Deb2002NSGA2}); 2) Pareto-layer Thompson Sampling methods (\textsc{ParetoTS}, \textsc{ParetoTS+}~\cite{yahyaa2015thompson}); 3) Chebyshev scalarization methods (\textsc{ChebyshevUCB}, \textsc{ChebyshevUCB+}~\cite{Mandow2023ChebyshevMAB}); 4) linear and hypervolume scalarization methods (\textsc{ScalarUCB}~\cite{auer2002finite}, \textsc{ScalarUCB-RandW}~\cite{paria2020flexible}, \textsc{HVScalarUCB}~\cite{zhang2020random}, \textsc{HVScalarUCB+}~\cite{Zhang2023HypervolumeRegret}); and 5) \textsc{RandomK}, which selects a slate uniformly at random. The \texttt{+} suffix marks our extension when both versions share a citation.} 
All methods share the same initialization scheme (forced round-robin sampling until each arm has been pulled at least $\texttt{min\_pulls}$ times) to avoid degenerate early behavior.

Implementation details, including UCB bonuses, posterior parameterizations, and tie-breaking rules, are as follows:

\ttn{
\begin{enumerate}
    \item \textsc{ParetoUCB} family
    \begin{itemize}
        \item \textsc{ParetoUCB}~\cite{Drugan2013MOBandits} computes UCB vectors $U_i = \hat{\mu}_i + \beta_i \mathbf{1}$ and selects $k$ arms by iterating Pareto layers on $U$, using a $\sum_j U_{ij}$ tie-break within each layer. The faithful variant (\textsc{ParetoUCB}) uses the original confidence term $\beta_i = \sqrt{2 \log(t \cdot (d \cdot |\hat{\mathcal{F}}|)^{1/4}) / N_i}$ where $|\hat{\mathcal{F}}|$ is the empirical Pareto front size, with uniform random selection within each layer.
        \item \textsc{ParetoUCB-Div}~\cite{Drugan2013MOBandits} extends \textsc{ParetoUCB} by replacing the $\sum U$ tie-break with a farthest-point (maximin $\ell_\infty$) diversity criterion: each slot greedily picks the candidate maximally distant from already-selected arms in UCB space.
        \item \textsc{ParetoUCB-Crowd}~\cite{Drugan2013MOBandits, Deb2002NSGA2} replaces the tie-break with a crowding distance~\cite{Deb2002NSGA2}: within each Pareto layer, arms are ranked by their normalized inter-neighbor gap across objectives, favoring spread along the frontier.
    \end{itemize}
    \item \textsc{ParetoTS} family
     \begin{itemize}
       \item \textsc{ParetoTS}~\cite{yahyaa2015thompson} maintains a Gaussian posterior per arm and coordinate. At each round it samples $\theta_i \sim \mathcal{N}(\hat{\mu}_i, \sigma_i^2 \mathbf{I})$ with posterior standard deviation $\sigma_i = \sigma_{\mathrm{obs}} / \sqrt{N_i}$ (conjugate Gaussian), builds Pareto layers on the sampled vectors, and selects $k$ arms by uniform random sampling within each layer. The only change from the original is the top-$k$ extension.
        \item \textsc{ParetoTS}$^+$~\citep{yahyaa2015thompson} uses a heuristic posterior $\sigma_i^2 = (\sigma_{\mathrm{prior}}^2 + \sigma_{\mathrm{obs}}^2) / N_i$ that decays more slowly, combined with a $\sum_j \theta_{ij}$ tie-break to encourage diversity within each Pareto layer.
    \end{itemize}
    \item \textsc{ChebyshevUCB} family
    \begin{itemize}
         \item \textsc{ChebyshevUCB}~\cite{Mandow2023ChebyshevMAB} follows Algorithm C2 of \cite{Mandow2023ChebyshevMAB}: a set of $S$ scalarization functions with weights spread uniformly on the simplex is precomputed, a function $f^j$ is drawn uniformly at each round, and arms are scored by $\min_\ell \{ w^j_\ell (U_{i\ell} - z_\ell) \}$ where $z$ is an estimated nadir point and $U_i = \hat{\mu}_i + \beta_i \mathbf{1}$ is the UCB vector.
        \item \textsc{ChebyshevUCB}$^+$~\cite{Mandow2023ChebyshevMAB} fixes $w = \mathbf{1}/d$ uniformly instead of drawing $f^j$ randomly, and uses the standard UCB bonus $\beta_i = \sqrt{\eta \log(ndT^2) / (2N_i)}$.
    \end{itemize}
    \item \textsc{ScalarUCB} family
    \begin{itemize}
    \item \textsc{ScalarUCB}~\cite{auer2002finite} reduces the vector reward to a scalar via a fixed uniform linear scalarization $w = \mathbf{1}/d$, maintains a scalar mean estimate per arm, and selects the top-$k$ arms by UCB index $\hat{\mu}_i + \beta_i$ where $\beta_i = \sqrt{\eta \log(nT^2) / (2 N_i)}$.    
    \item \textsc{ScalarUCB-RandW}~\cite{paria2020flexible} draws a fresh weight vector $w \sim \mathrm{Dirichlet}(\mathbf{1})$ at each round, computes UCB vectors $U_i = \hat{\mu}_i + \beta_i \mathbf{1}$ coordinate-wise, and selects the top-$k$ arms by score $w^\top U_i$.   
    \item \textsc{HVScalarUCB}~\cite{zhang2020random} draws $\lambda \sim S^{d-1}_+$ (positive unit sphere) at each round and scores arms by the hypervolume scalarization $s_\lambda(U_i) = \bigl(\min_\ell \max(0, U_{i\ell}/\lambda_\ell)\bigr)^d$ (Lemma~5 of \cite{zhang2020random}), selecting the top-$k$ arms by score.
    \item \textsc{HVScalarUCB}$^+$~\cite{Zhang2023HypervolumeRegret} draws $k$ independent directions $\lambda^{(1)}, \ldots, \lambda^{(k)} \sim S^{d-1}_+$ per round and assigns one arm per direction: $a^{(j)} = \arg\max_{i \notin \{a^{(1)},\ldots,a^{(j-1)}\}} \min_\ell (U_{i\ell} - r_\ell)/\lambda^{(j)}_\ell$, where $r$ is the reference point. This top-$k$ adaptation is our own extension of the directional intuition of Lemma~5 in \cite{Zhang2023HypervolumeRegret}.
    \end{itemize}
    \item \textsc{RandomK} selects a slate of $k$ arms uniformly at random each round, without any learning.
\end{enumerate}
}



\subsection{Results}
\label{sec:exp:results}
Overall, \textit{THV-UCB} achieves the lowest cumulative $\alpha$-regret and the highest hypervolume in all four front geometries (linear, convex, concave, and clusters) and all dimensions from $d=2$ to $d=5$, with a margin that generally increases with $d$ (See Tables~\ref{tab:synthetic-summary}, Figures~\ref{fig:synthetic-fronts-linear},~\ref{fig:synthetic-fronts-clusters},~\ref{fig:synthetic-fronts-concave},~\ref{fig:synthetic-fronts-convex}, and all the results (for $d>2$) in the appendix and in our GitHub repository \url{https://github.com/ngutowski/topk-pareto-bandits}).

Figures~\ref{fig:synthetic-fronts-linear}~to~\ref{fig:synthetic-fronts-convex}, report hypervolume
$\mathrm{HV}(S_t)$ trajectories (moving average, $w=50$) with $95\%$ CIs across $10$ seeds, comparing THV-UCB (ours) against the best representative per baseline family, for each synthetic Pareto front geometry.


\begin{table*}
\centering
\small
\setlength{\tabcolsep}{4pt}
\resizebox{\textwidth}{!}{
\begin{tabular}{lccccccccc}
\toprule
& & \multicolumn{2}{c}{\textbf{Clusters}} & \multicolumn{2}{c}{\textbf{Concave}} & \multicolumn{2}{c}{\textbf{Convex}} & \multicolumn{2}{c}{\textbf{Linear}} \\
\cmidrule(lr){3-4}\cmidrule(lr){5-6}\cmidrule(lr){7-8}\cmidrule(lr){9-10}
\textbf{Algorithm} & \textbf{Fid.}
& HV$_{\text{last100}}$ & Final Regret
& HV$_{\text{last100}}$ & Final Regret
& HV$_{\text{last100}}$ & Final Regret
& HV$_{\text{last100}}$ & Final Regret \\
\midrule
THV-UCB (ours)
  & ---
  & \textbf{0.6556}{\scriptsize$\pm$0.001} & \textbf{5.5}{\scriptsize$< 0.001$}
  & \textbf{0.6638}{\scriptsize$\pm$0.002} & \textbf{6.2}{\scriptsize$< 0.001$}
  & \textbf{0.8536}{\scriptsize$\pm$0.002} & \textbf{7.2}{\scriptsize$< 0.001$}
  & \textbf{0.2879}{\scriptsize$\pm$0.002} & \textbf{2.7}{\scriptsize$< 0.001$} \\
\addlinespace[2pt]
ParetoUCB~\citep{Drugan2013MOBandits}
  & \faithful
  & 0.4438{\scriptsize$\pm$0.082} & 130.0{\scriptsize$\pm$108.3}
  & 0.4711{\scriptsize$\pm$0.072} & 105.0{\scriptsize$\pm$98.5}
  & 0.7169{\scriptsize$\pm$0.059} & 38.8{\scriptsize$\pm$43.2}
  & 0.1515{\scriptsize$\pm$0.040} & 107.1{\scriptsize$\pm$64.7} \\
ParetoUCB\textsuperscript{+}~\citep{Drugan2013MOBandits}
  & \extension
  & 0.6330{\scriptsize$\pm$0.005} & 28.4{\scriptsize$\pm$5.6}
  & 0.5770{\scriptsize$\pm$0.023} & 31.3{\scriptsize$\pm$5.2}
  & 0.8170{\scriptsize$\pm$0.015} & 7.2{\scriptsize$< 0.001$}
  & 0.1620{\scriptsize$\pm$0.037} & 69.5{\scriptsize$\pm$46.2} \\
ParetoUCB-Div~\citep{Drugan2013MOBandits}
  & \extension
  & 0.6502{\scriptsize$\pm$0.007} & 9.4{\scriptsize$\pm$2.1}
  & 0.6064{\scriptsize$\pm$0.006} & 6.6{\scriptsize$\pm$0.7}
  & 0.8281{\scriptsize$\pm$0.006} & \textbf{7.2}{\scriptsize$< 0.001$}
  & 0.2115{\scriptsize$\pm$0.009} & 5.8{\scriptsize$\pm$4.8} \\
ParetoUCB-Crowd~\citep{Drugan2013MOBandits,Deb2002NSGA2}
  & \extension
  & 0.6384{\scriptsize$\pm$0.011} & 19.8{\scriptsize$\pm$7.8}
  & 0.5712{\scriptsize$\pm$0.015} & 18.0{\scriptsize$\pm$7.7}
  & 0.8290{\scriptsize$\pm$0.006} & 7.3{\scriptsize$< 0.001$}
  & 0.1303{\scriptsize$\pm$0.008} & 99.5{\scriptsize$\pm$19.1} \\
\addlinespace[2pt]
ParetoTS~\citep{yahyaa2015thompson}
  & \faithful
  & 0.4871{\scriptsize$\pm$0.069} & 51.8{\scriptsize$\pm$55.7}
  & 0.5180{\scriptsize$\pm$0.056} & 33.4{\scriptsize$\pm$38.0}
  & 0.7528{\scriptsize$\pm$0.044} & 9.4{\scriptsize$\pm$3.8}
  & 0.1901{\scriptsize$\pm$0.035} & 40.0{\scriptsize$\pm$40.0} \\
ParetoTS\textsuperscript{+}~\citep{yahyaa2015thompson}
  & \extension
  & 0.6293{\scriptsize$\pm$0.015} & 12.5{\scriptsize$\pm$8.9}
  & 0.5788{\scriptsize$\pm$0.025} & 12.5{\scriptsize$\pm$8.9}
  & 0.7836{\scriptsize$\pm$0.025} & 8.4{\scriptsize$\pm$2.0}
  & 0.1623{\scriptsize$\pm$0.040} & 85.6{\scriptsize$\pm$61.1} \\
\addlinespace[2pt]
ChebyshevUCB~\citep{Mandow2023ChebyshevMAB}
  & \faithful
  & 0.4763{\scriptsize$\pm$0.066} & 72.0{\scriptsize$\pm$62.7}
  & 0.5017{\scriptsize$\pm$0.063} & 48.1{\scriptsize$\pm$47.7}
  & 0.7185{\scriptsize$\pm$0.033} & 23.7{\scriptsize$\pm$16.4}
  & 0.2016{\scriptsize$\pm$0.027} & 38.4{\scriptsize$\pm$34.5} \\
ChebyshevUCB\textsuperscript{+}~\citep{Mandow2023ChebyshevMAB}
  & \extension
  & 0.6286{\scriptsize$\pm$0.002} & 10.3{\scriptsize$\pm$2.0}
  & 0.5514{\scriptsize$\pm$0.012} & 24.2{\scriptsize$\pm$3.2}
  & 0.7527{\scriptsize$\pm$0.014} & 7.3{\scriptsize$< 0.001$}
  & 0.2266{\scriptsize$\pm$0.005} & 10.0{\scriptsize$\pm$5.1} \\
\addlinespace[2pt]
HVScalarUCB~\citep{zhang2020random}
  & \faithful
  & 0.4898{\scriptsize$\pm$0.083} & 86.7{\scriptsize$\pm$88.6}
  & 0.5086{\scriptsize$\pm$0.082} & 80.1{\scriptsize$\pm$82.3}
  & 0.7216{\scriptsize$\pm$0.039} & 13.8{\scriptsize$\pm$5.5}
  & 0.2105{\scriptsize$\pm$0.035} & 46.8{\scriptsize$\pm$49.3} \\
HVScalarUCB\textsuperscript{+}~\citep{Zhang2023HypervolumeRegret}
  & \faithful
  & 0.5728{\scriptsize$\pm$0.050} & 15.6{\scriptsize$\pm$16.7}
  & 0.5863{\scriptsize$\pm$0.043} & 14.0{\scriptsize$\pm$13.1}
  & 0.7647{\scriptsize$\pm$0.047} & 8.9{\scriptsize$\pm$3.1}
  & 0.2463{\scriptsize$\pm$0.017} & 6.7{\scriptsize$\pm$6.6} \\
ScalarUCB~\citep{auer2002finite}
  & \extension
  & 0.6340{\scriptsize$\pm$0.005} & 10.9{\scriptsize$\pm$6.1}
  & 0.5733{\scriptsize$\pm$0.025} & 27.2{\scriptsize$\pm$20.7}
  & 0.7731{\scriptsize$\pm$0.022} & 14.0{\scriptsize$\pm$8.1}
  & 0.1584{\scriptsize$\pm$0.040} & 95.0{\scriptsize$\pm$62.1} \\
ScalarUCB-RandW~\citep{paria2020flexible}
  & \extension
  & 0.4536{\scriptsize$\pm$0.069} & 66.4{\scriptsize$\pm$42.5}
  & 0.4734{\scriptsize$\pm$0.056} & 50.1{\scriptsize$\pm$22.9}
  & 0.7347{\scriptsize$\pm$0.036} & 7.3{\scriptsize$< 0.001$}
  & 0.0733{\scriptsize$\pm$0.040} & 180.1{\scriptsize$\pm$62.8} \\
\addlinespace[2pt]
RandomK
  & ---
  & 0.2225{\scriptsize$\pm$0.105} & 424.0{\scriptsize$\pm$181.6}
  & 0.2361{\scriptsize$\pm$0.111} & 405.7{\scriptsize$\pm$190.6}
  & 0.3878{\scriptsize$\pm$0.165} & 411.2{\scriptsize$\pm$256.6}
  & 0.0922{\scriptsize$\pm$0.037} & 185.9{\scriptsize$\pm$69.3} \\
\bottomrule
\end{tabular}
}\captionsetup{width=\textwidth}
\captionof{table}{Summary over four synthetic fronts ($d=2$, $n=36$, $k=3$,  $\sigma=0.05$, $T=2000$, 10 seeds).
\textbf{Fid.} indicates fidelity to the cited work:
\faithful~= faithful top-$k$ adaptation (core mechanism unchanged);
\extension~= our extension (modified core, tie-break, or diversity mechanism not in the original).
\textsuperscript{+} denotes our variant when both versions share a citation.
Final Regret values are cumulative $\alpha$-regret $\pm$ 95\% CI.
} \vspace{-0.4cm}

\label{tab:synthetic-summary}
\end{table*}

\ttn{The linear front is the most discriminative, with THV-UCB's advantage over the closest baseline widening from +27\% at $d=2$ to +79\% at $d=5$: linear fronts require uniform simplex coverage, which single-direction scalarization methods increasingly fail to achieve at higher $d$, while THV-UCB's greedy hypervolume gain
naturally spreads the slate across the entire front (See Table~\ref{tab:synthetic-summary} and Figure~\ref{fig:synthetic-fronts-linear}).}
\begin{figure}[h]
        \includegraphics[width=\columnwidth]{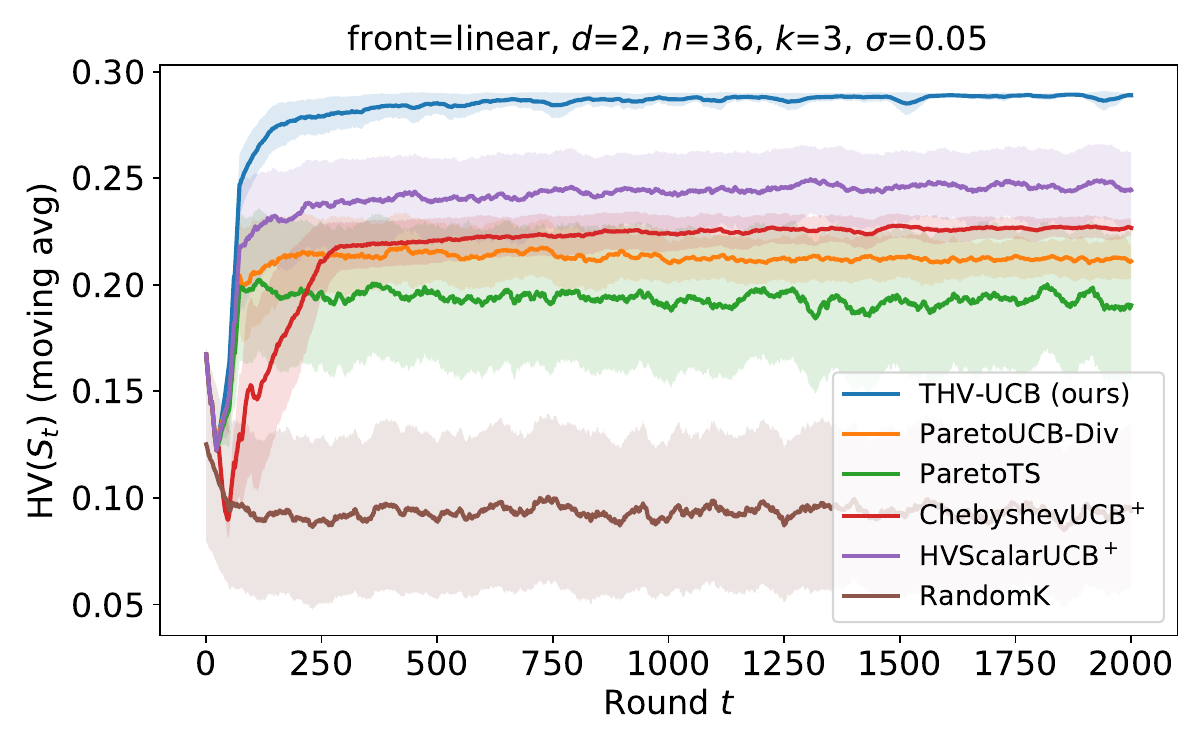}\vspace{-0.4cm}
        \caption{HV trajectories, linear front.}\vspace{-0.4cm}
        \label{fig:synthetic-fronts-linear}
    \end{figure}

\ttn{On the clusters front, THV-UCB consistently outperforms all baselines, confirming that the set-level hypervolume objective is essential when the front has disconnected regions: no single-direction scalarization can reliably cover both clusters within a single round (See Table~\ref{tab:synthetic-summary} and Figure~\ref{fig:synthetic-fronts-clusters}).}
    \begin{figure}[h]        \includegraphics[width=\columnwidth]{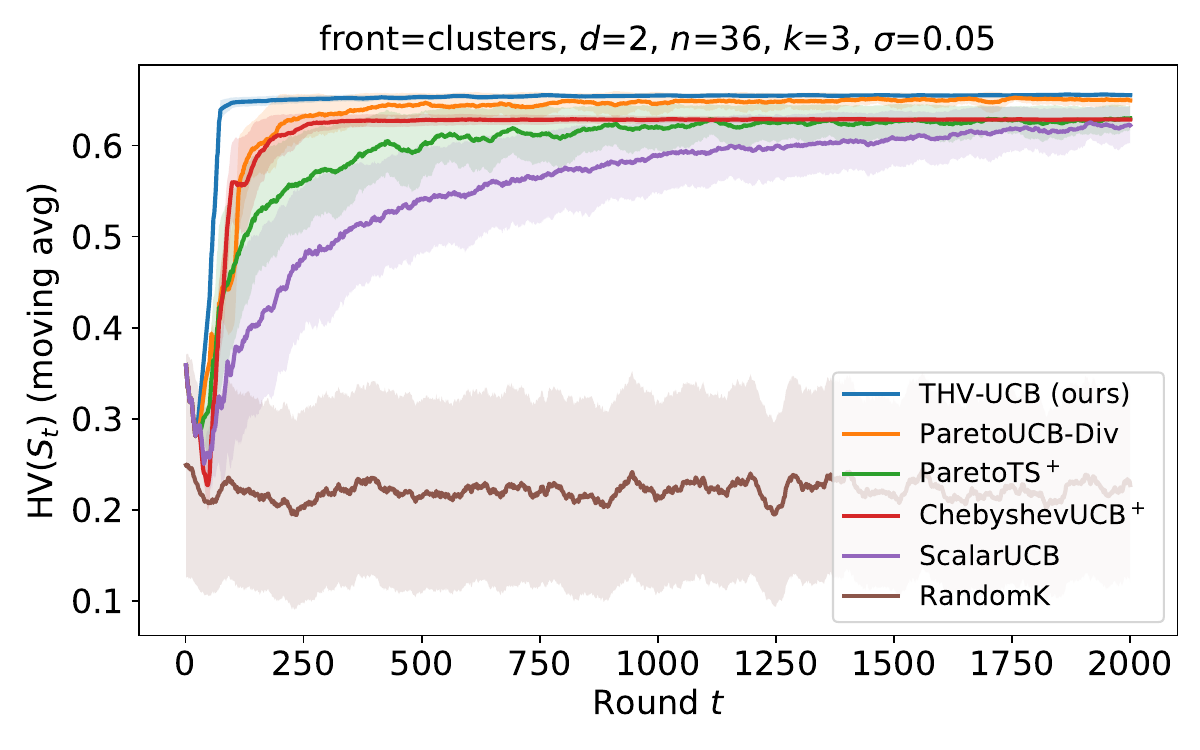}\vspace{-0.4cm}
        \caption{HV trajectories, clusters front.}\vspace{-0.4cm}
        \label{fig:synthetic-fronts-clusters}
    \end{figure}

\ttn{On the concave and convex fronts, THV-UCB leads throughout but faces stronger competition: ChebyshevUCB$^+$ is the closest competitor on concave at $d=3,4$, while ScalarUCB and ParetoTS$^+$ are competitive on convex across dimensions (See Table~\ref{tab:synthetic-summary}, and Figures~\ref{fig:synthetic-fronts-concave} and~\ref{fig:synthetic-fronts-convex}).}
\begin{figure}[h]
        \includegraphics[width=\columnwidth]{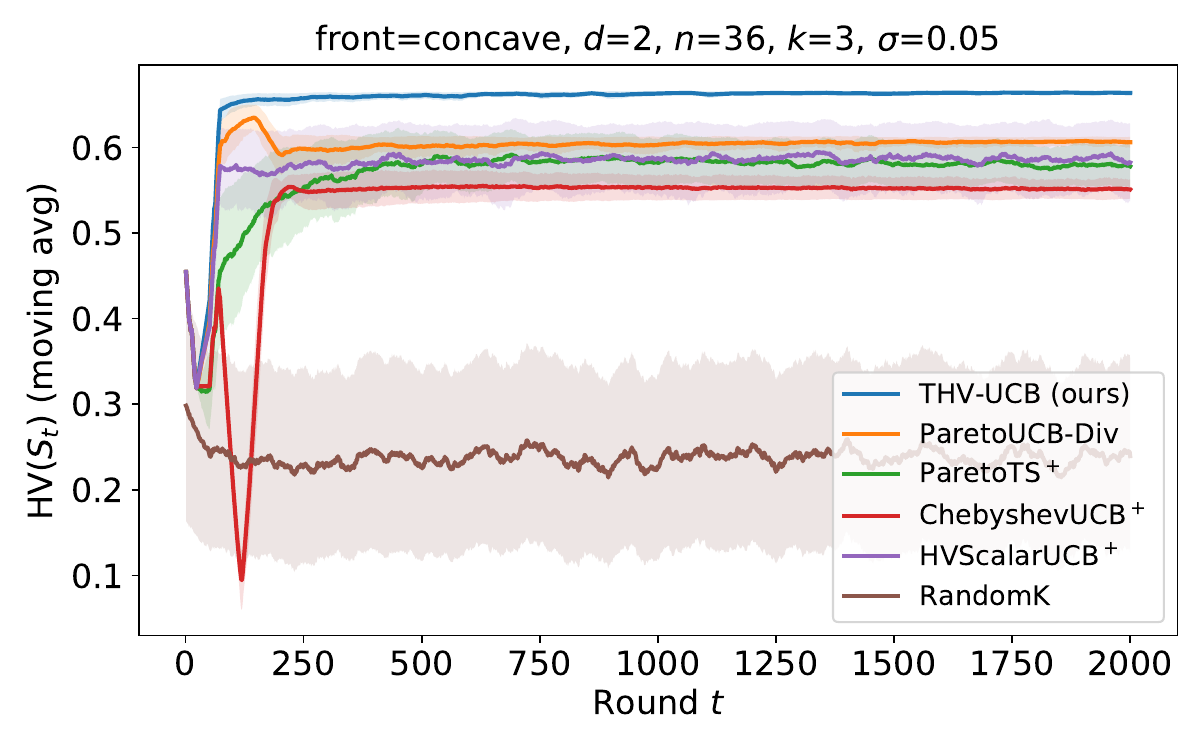}\vspace{-0.4cm}
        \caption{HV trajectories, concave front.}\vspace{-0.4cm}
        \label{fig:synthetic-fronts-concave}
    \end{figure}
   \begin{figure}[h]
        \includegraphics[width=\columnwidth]{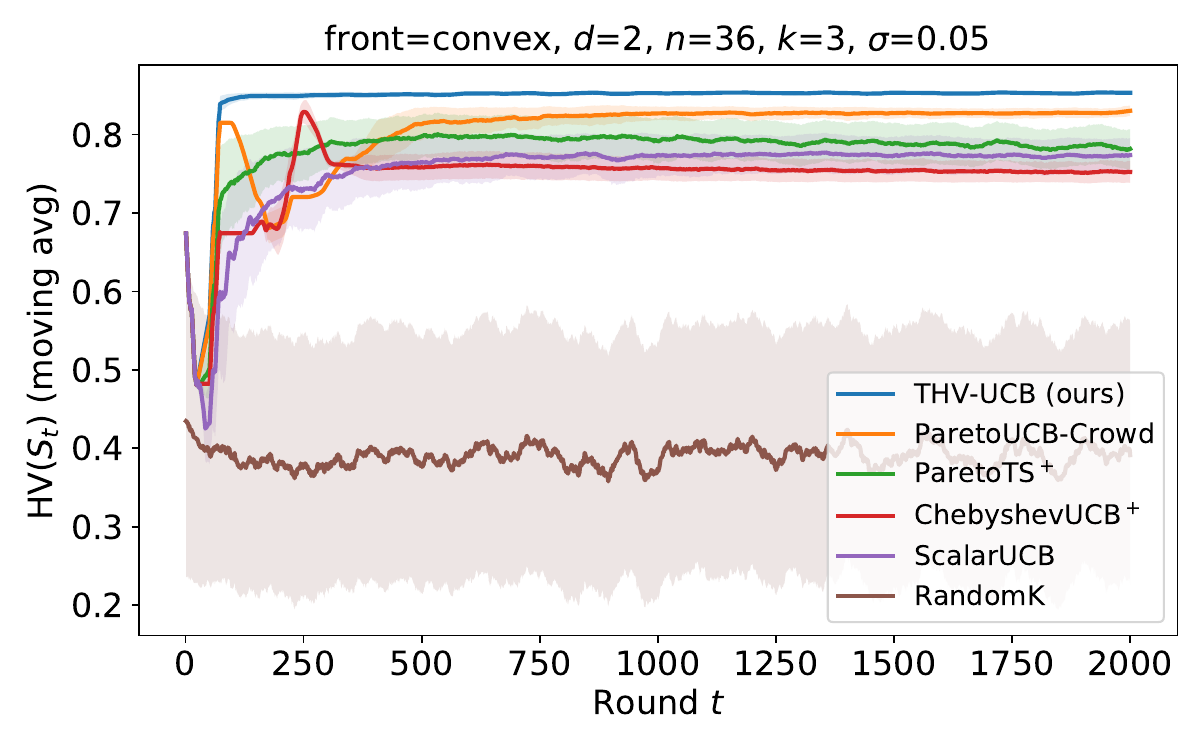}\vspace{-0.4cm}
        \caption{HV trajectories, convex front.}\vspace{-0.4cm}
        \label{fig:synthetic-fronts-convex}
    \end{figure}
\ttn{Among baselines, no single method dominates across all settings. ChebyshevUCB$^+$~\citep{Mandow2023ChebyshevMAB} is the strongest competitor at $d=3$ and $4$, particularly on concave and linear fronts. HVScalarUCB$^+$~\citep{Zhang2023HypervolumeRegret} becomes increasingly competitive on linear as $d$ grows, reaching second place at $d=4$ ($0.0532$) and $d=5$ ($0.0053$). ParetoUCB$^+$~\citep{Drugan2013MOBandits} and ScalarUCB~\citep{auer2002finite} perform well on smooth fronts (convex, concave) at low $d$ but degrade on linear and clusters. ParetoUCB-Div~\cite{Drugan2013MOBandits} and ParetoUCB-Crowd~\cite{Deb2002NSGA2,Drugan2013MOBandits} add diversity heuristics that help on clusters at $d=2$ but loose their advantage at higher $d$. ScalarUCB-RandW~\citep{paria2020flexible} is consistently among the weakest baselines due to the high variance induced by random scalarization weights, and RandomK performs worst in all settings as expected.}

\subsection{Statistical tests}
\ttn{Paired Wilcoxon tests on the tightest margins confirm significance
($p<0.001$, $10/10$ seeds, Cohen's $d$ from $1.9$ to $44.5$, See Table~\ref{tab:stat-tests} in the appendix).}

\section{Conclusion}
\label{sec:conclusion}
We introduced \emph{Top-$k$ Pareto Bandits}, where an agent repeatedly selects a size-$k$ slate under semi-bandit feedback, evaluated by dominated hypervolume coverage of the Pareto boundary, and proposed \textit{THV-UCB}, which greedily maximizes optimistic marginal hypervolume gain under safe coordinate-wise pruning. We established a gap-free $\tilde{O}\!\left(d\sqrt{nkT}\right)$ bound valid on every instance and a gap-dependent $\tilde{O}\!\left(nk^{2.5}/\Delta_{\min}\right)$ bound that is polylogarithmic in $T$ on well-separated instances; tightening these dependencies and establishing matching lower bounds remain open. Empirically, \textit{THV-UCB} outperforms state-of-the-art baselines across all tested geometries and dimensions $d\in\{2,\dots,5\}$, with the margin widening as $d$ increases, supporting hypervolume-driven slate selection for applications such as recommender systems, portfolio management, or automated decision support.

\bibliography{aaai2027.bib}


\clearpage

\input{supplementary.tex}

\end{document}

%% file: supplementary.tex
\onecolumn

\begin{quote}
\small
This appendix provides the full proofs of
Theorems~\ref{thm:gapfree}--\ref{thm:combined} (Upper-bound Proofs, below),
additional experimental results for $d\in\{3,4,5\}$ (Experimental
details), and a summary of notation (Notation summary). 

\end{quote}

\appendix
\section{Upper-bound Proofs}
\label{app:upper-proofs}

This appendix establishes the sublinear $\alpha$-regret guarantee for \textit{THV-UCB} announced in the Regret Analysis section. We prove two complementary statements:
\begin{itemize}
    \item a \emph{gap-free} bound (Theorem~\ref{thm:gap-free}), of order $\tilde{O}(\sqrt{nkT})$, valid on \emph{every} instance regardless of how close the arms are to each other;
    \item a \emph{gap-dependent} bound (Theorem~\ref{thm:gap-dependent}), of order $\tilde{O}(nk^{2.5}/\Delta_{\min})$, which becomes polylogarithmic in $T$ as soon as the instance is well separated.
\end{itemize}
Both proofs share the same first steps (Preliminaries through Lipschitz control of the optimism error), which reduce $\bar R_T$ to a sum of confidence radii, in~\eqref{eq:width-sum-bound}. They diverge only in \emph{how} this sum is controlled: 
the gap-free bound 
(see Section \emph{A gap-free bound}) treats every pull anonymously via Cauchy--Schwarz over the whole horizon; the gap-dependent bound (See Section \emph{A gap-dependent bound}) opens up the greedy construction stage by stage, 
tracks the actual comparisons $\Delta_t^{\mathrm{UCB}}(\cdot\mid S)$ performed by the algorithm, and applies the same Cauchy--Schwarz idea \emph{locally}, to each stage's deviation rounds. We discuss what each proof technique does, and does not, use about \textit{THV-UCB} in the discussion at the end of this appendix.

\subsection{Preliminaries}
\label{app:prelim}

\begin{lemma}[Monotonicity and submodularity of hypervolume]
\label{lem:submod}
For any fixed reference point $\mathbf r$ dominated by all achievable means, the set function $S\mapsto \mathrm{HV}(S)$ is monotone (non-decreasing) and submodular.
\end{lemma}

\begin{proof}
Recall $\mathrm{HV}(S) = \lambda_d\bigl(\bigcup_{i\in S}\prod_{j=1}^d[r_j,\mu_{i,j}]\bigr)$.

\emph{Monotonicity.} If $S\subseteq T$, the union of boxes only grows, so $\mathrm{HV}(S)\le\mathrm{HV}(T)$.

\emph{Submodularity.} Let $A \subseteq B$ and $i \notin B$. Write $V_A=\bigcup_{j\in A}\prod_\ell[r_\ell,\mu_{j,\ell}]$ and $V_i=\prod_\ell[r_\ell,\mu_{i,\ell}]$. Then
\[
\mathrm{HV}(A\cup\{i\})-\mathrm{HV}(A)=\lambda_d(V_i\setminus V_A),
\qquad
\mathrm{HV}(B\cup\{i\})-\mathrm{HV}(B)=\lambda_d(V_i\setminus V_B).
\]
Since $V_A\subseteq V_B$, $V_i\setminus V_B\subseteq V_i\setminus V_A$, hence submodularity.
\end{proof}

\begin{remark}
\label{rem:ucb-submod}
The same argument, applied verbatim with $u_{i,j}(t)$ in place of $\mu_{i,j}$ (the proof never uses anything about $\mu_{i,j}$ beyond it being a real number), shows that for every fixed $t$, $S\mapsto \mathrm{HV}_t^{\mathrm{UCB}}(S)$ 
is also monotone and submodular.
\end{remark}

\begin{lemma}[Coordinate-wise concentration]
\label{lem:conc}
Assume each coordinate is $\eta$-sub-Gaussian. Then with probability at least $1-\delta$, for all $i,j,t$,
\[
\bigl|\widehat\mu_{i,j}(t)-\mu_{i,j}\bigr|\le \beta_i(t)=\sqrt{\frac{2\eta\log(ndt^2/\delta)}{\max\{1,N_i(t-1)\}}}.
\]
\end{lemma}

\begin{proof}
Fix $i,j$ and a pull count $s=N_i(t-1)\ge1$. By the sub-Gaussian Hoeffding bound, $\mathbb P\bigl(|\widehat\mu_{i,j}(t)-\mu_{i,j}|>\sqrt{2\eta\log(x)/s}\bigr)\le 2/x$. Taking $x=ndt^2/\delta$ and a union bound over $i\in[n]$, $j\in[d]$, and $t\in[T]$ (using $\sum_{t\ge1}t^{-2}\le 2$) gives total failure probability $O(\delta)$; absorbing the constant into $\delta$ yields the stated bound.
\end{proof}

We write $\mathcal E$ for the event on which Lemma~\ref{lem:conc} holds, so $\mathbb P(\mathcal E)\ge1-\delta$. After the initialization phase of Algorithm~\ref{alg:thvucb} (which lasts at most $t_0=\lceil nm_0/k\rceil$ rounds), every arm has $N_i(t-1)\ge m_0\ge1$, so the $\max\{1,\cdot\}$ in $\beta_i(t)$ is never active for $t>t_0$; we drop it below for readability.

\subsection{A shared greedy-approximation theorem}
\label{app:greedy-generic}

The following classical fact (Nemhauser, Wolsey \& Fisher, 1978) is invoked twice in this appendix, for two different monotone submodular functions; we state it once to avoid duplicating the argument and, crucially, to keep visually distinct the two objects it produces.

\begin{fact}[Greedy approximation for monotone submodular maximization]
\label{thm:greedy-generic}
Let $f:2^{[n]}\to\mathbb R$ be monotone and submodular, and let $G=\{g_1,\dots,g_k\}$ be built by greedy maximization of marginal $f$-gain under the cardinality constraint $k$ (i.e.\ $g_\ell\in\arg\max_{i\notin G_{\ell-1}} f(G_{\ell-1}\cup\{i\})-f(G_{\ell-1})$, $G_\ell=G_{\ell-1}\cup\{g_\ell\}$, $G_0=\emptyset$, $G=G_k$). Then
\[
f(G)\;\ge\;\Bigl(1-\tfrac1e\Bigr)\max_{|S|=k}f(S)\;=\;\alpha\max_{|S|=k}f(S).
\]
\end{fact}

We will apply Fact~\ref{thm:greedy-generic} to two different functions:
\begin{itemize}
    \item $f=\mathrm{HV}_t^{\mathrm{UCB}}$ (monotone submodular by the previous remark, with $G=S_t$ the set actually built by \textit{THV-UCB} at round $t$ — this yields Corollary~\ref{lem:greedy} below, used in the reduction step;
    \item $f=\mathrm{HV}$ \textbf{directly on the true means} — a purely deterministic statement requiring no concentration event — with $G=G^\star$, a benchmark sequence introduced 
    in Section \emph{A gap-dependent bound} and \textbf{distinct from the regret's true optimum} $S^\star$.
\end{itemize}

\begin{corollary}[Greedy approximation under optimism]
\label{lem:greedy}
For every round $t$, $\mathrm{HV}_t^{\mathrm{UCB}}(S_t)\ge \alpha\,\mathrm{HV}_t^{\mathrm{UCB}}(S^{\mathrm{opt}}_t)$, where $S^{\mathrm{opt}}_t\in\arg\max_{|S|=k}\mathrm{HV}_t^{\mathrm{UCB}}(S)$.
\end{corollary}

\subsection{Optimism and reduction to optimism error}
\label{app:reduction}

Recall $u_{i,j}(t)=\widehat\mu_{i,j}(t-1)+\beta_i(t)$ and
\[
\mathrm{HV}_t^{\mathrm{UCB}}(S)=\lambda_d\Bigl(\bigcup_{i\in S}[r_1,u_{i,1}(t)]\times\cdots\times[r_d,u_{i,d}(t)]\Bigr).
\]
On $\mathcal E$, $\mu_{i,j}\le u_{i,j}(t)$ for all $i,j$, so by monotonicity $\mathrm{HV}(S)\le \mathrm{HV}_t^{\mathrm{UCB}}(S)$ for every $S$. Let $S^\star\in\arg\max_{|S|=k}\mathrm{HV}(S)$, with $V^\star=\mathrm{HV}(S^\star)$ — this is the benchmark appearing in the regret definition~\eqref{eq:regret-alpha}, and the only role $S^\star$ plays in this appendix. Since $S^{\mathrm{opt}}_t$ maximizes $\mathrm{HV}_t^{\mathrm{UCB}}$ over \emph{all} size-$k$ sets, it dominates $S^\star$ under this score, so by Corollary~\ref{lem:greedy},
\[
\mathrm{HV}_t^{\mathrm{UCB}}(S_t)\;\ge\;\alpha\,\mathrm{HV}_t^{\mathrm{UCB}}(S^{\mathrm{opt}}_t)\;\ge\;\alpha\,\mathrm{HV}_t^{\mathrm{UCB}}(S^\star).
\]
Therefore, on $\mathcal E$,
\begin{equation}
\bar r_t=\alpha\,V^\star-\mathrm{HV}(S_t)\;\le\;\alpha\,\mathrm{HV}_t^{\mathrm{UCB}}(S^\star)-\mathrm{HV}(S_t)\;\le\;\mathrm{HV}_t^{\mathrm{UCB}}(S_t)-\mathrm{HV}(S_t).
\label{eq:key-bound-ptwise}
\end{equation}
Summing over $t\in[T]$,
\begin{equation}
\label{eq:key-bound}
\bar R_T\;\le\;\sum_{t=1}^T\bigl(\mathrm{HV}_t^{\mathrm{UCB}}(S_t)-\mathrm{HV}(S_t)\bigr)\qquad\text{on }\mathcal E.
\end{equation}

\subsection{Lipschitz control of the optimism error}
\label{app:lipschitz}

\begin{lemma}[Lipschitz bound for hypervolume under coordinate-wise shifts]
\label{lem:lip}
There is a constant $C_d$ (one may take $C_d\le d$ when rewards lie in $[0,1]^d$ and $\mathbf r=\mathbf 0$) such that, on $\mathcal E$, for every round $t$ and \textbf{every finite set $S\subseteq[n]$} (not necessarily $S_t$),
\[
\mathrm{HV}_t^{\mathrm{UCB}}(S)-\mathrm{HV}(S)\;\le\; C_d\sum_{i\in S}\beta_i(t).
\]
\end{lemma}

\begin{proof}
Write $S=\{i_1,\dots,i_m\}$. Define the intermediate hypervolume $A_\ell$ where arms $i_1,\dots,i_\ell$ use optimistic vectors $\mathbf u_{i_p}(t)$ and arms $i_{\ell+1},\dots,i_m$ use true means $\boldsymbol\mu_{i_p}$, so $\mathrm{HV}_t^{\mathrm{UCB}}(S)-\mathrm{HV}(S)=\sum_{\ell=1}^m(A_\ell-A_{\ell-1})$. Each term changes only arm $i_\ell$ from $\boldsymbol\mu_{i_\ell}$ to $\mathbf u_{i_\ell}(t)=\boldsymbol\mu_{i_\ell}+\boldsymbol\delta_{i_\ell}$ with $0\le\delta_{i_\ell,j}\le\beta_{i_\ell}(t)$. The added hypervolume is contained in a union of $d$ axis-aligned slabs of width at most $\beta_{i_\ell}(t)$ in one coordinate and at most $1$ in the others, so $A_\ell-A_{\ell-1}\le d\,\beta_{i_\ell}(t)$. Summing over $\ell$ gives the claim with $C_d=d$. Crucially, this argument never uses $|S|=k$ nor any property of how $S$ was selected, hence it holds for an arbitrary finite $S$.
\end{proof}

Combining ~\eqref{eq:key-bound} with Lemma~\ref{lem:lip} applied to $S=S_t$,
\begin{equation}
\label{eq:width-sum-bound}
\bar R_T\;\le\; C_d\sum_{t=1}^T\sum_{i\in S_t}\beta_i(t)\qquad\text{on }\mathcal E.
\end{equation}
The rest of the proof bounds the right-hand side of \eqref{eq:width-sum-bound} in two ways.

\subsection{A gap-free bound}
\label{app:gap-free}

Let $N_i(T)=\sum_{t=1}^T\mathbf 1\{i\in S_t\}$. Reordering \eqref{eq:width-sum-bound} by arm,
\[
\sum_{t=1}^T\sum_{i\in S_t}\beta_i(t)
=\sum_{i=1}^n\sum_{s=1}^{N_i(T)}\sqrt{\frac{2\eta\log(ndT^2/\delta)}{s}}
\;\le\; 2\sqrt{2\eta\log(ndT^2/\delta)}\sum_{i=1}^n\sqrt{N_i(T)},
\]
using $\sum_{s=1}^m s^{-1/2}\le 2\sqrt m$. Since exactly $k$ arms are pulled per round, $\sum_i N_i(T)=kT$, and Cauchy--Schwarz gives $\sum_i\sqrt{N_i(T)}\le\sqrt{nkT}$. Hence, on $\mathcal E$, $\bar R_T\le O\bigl(C_d\sqrt{nkT\log T}\bigr)$.

\setcounter{theorem}{0}

\begin{theorem}[Gap-free bound (detailed version)]
\label{thm:gap-free}
On $\mathcal E$, $\bar R_T = O\bigl(C_d\sqrt{nkT\log T}\bigr)$.
\end{theorem}

\begin{remark}
This argument never used that $S_t$ is built by comparing marginal gains $\Delta_t^{\mathrm{UCB}}(i\mid S)$ across candidates: it only used $|S_t|=k$ and the parametric decay $\beta_i(t)=\Theta(1/\sqrt{N_i(t-1)})$. It is "gap-free" precisely because it is blind to how close arms are, hence to the selection mechanism itself.
\end{remark}

\subsection{A gap-dependent bound}
\label{app:gap-dependent}

The key idea: the algorithm's own greedy construction at round $t$ progressively learns a stage-by-stage benchmark sequence. At stage $\ell$, it must identify the best arm conditionally on the previously identified greedy prefix.

\subsubsection{The greedy optimal set}

Let $G^\star=\{i_1^\star,\dots,i_k^\star\}$ be a greedy construction of an optimal size-$k$ set for the \emph{true} hypervolume:
\[
i_\ell^\star\in\arg\max_{i\notin G^\star_{\ell-1}}\Delta(i\mid G^\star_{\ell-1}),
\qquad
G^\star_\ell=G^\star_{\ell-1}\cup\{i_\ell^\star\},\quad G^\star_0=\emptyset,
\]
with $\Delta(i\mid S)=\mathrm{HV}(S\cup\{i\})-\mathrm{HV}(S)$.

Assume the greedy maximizer is unique at every stage of $G^\star$'s construction, and define the stage-$\ell$ gaps relative to this benchmark:
\[
\Delta_\ell(i)=\Delta(i_\ell^\star\mid G^\star_{\ell-1})-\Delta(i\mid G^\star_{\ell-1})>0,\quad i\ne i_\ell^\star,
\qquad
\Delta_{\min}=\min_{\ell\in[k]}\min_{i\ne i_\ell^\star}\Delta_\ell(i).
\]

\begin{remark}
By submodularity (Lemma~\ref{lem:submod}), the sequence of greedy \emph{values} $v_\ell:=\Delta(i_\ell^\star\mid G^\star_{\ell-1})$ is non-increasing: for $i\notin G^\star_\ell$, $G^\star_{\ell-1}\subseteq G^\star_\ell$ gives $\Delta(i\mid G^\star_\ell)\le\Delta(i\mid G^\star_{\ell-1})$, so $v_{\ell+1}\le v_\ell$. This only constrains the \emph{best} marginal gain at each stage, not the runner-up: it does \emph{not} imply that the gap sequence $\gamma_\ell:=\min_{i\ne i_\ell^\star}\Delta_\ell(i)$ is monotonic in $\ell$. We therefore use the safe, uniform bound $\Delta_{\min}$ rather than a stage-wise one.
\end{remark}

\subsubsection{Deviation of optimistic marginal gains}

For $S\subseteq[n]$ and $i\notin S$, write $\Delta_t^{\mathrm{UCB}}(i\mid S)=\mathrm{HV}_t^{\mathrm{UCB}}(S\cup\{i\})-\mathrm{HV}_t^{\mathrm{UCB}}(S)$.

\begin{lemma}[Deviation of optimistic marginal gains]
\label{lem:marginal-dev-seq}
On $\mathcal E$, for every $S\subseteq[n]$ and $i\notin S$,
\[
\Delta_t^{\mathrm{UCB}}(i\mid S)\le \Delta(i\mid S)+C_d\sum_{j\in S\cup\{i\}}\beta_j(t), 
\qquad
\Delta_t^{\mathrm{UCB}}(i\mid S)\ge \Delta(i\mid S)-C_d\sum_{j\in S}\beta_j(t).
\]
\end{lemma}

\begin{proof}
By Lemma~\ref{lem:lip} applied to $S\cup\{i\}$ and to $S$, $\mathrm{HV}_t^{\mathrm{UCB}}(S\cup\{i\})\le\mathrm{HV}(S\cup\{i\})+C_d\sum_{j\in S\cup\{i\}}\beta_j(t)$, while optimism gives $\mathrm{HV}_t^{\mathrm{UCB}}(S)\ge\mathrm{HV}(S)$; subtracting gives the upper bound. Symmetrically, $\mathrm{HV}_t^{\mathrm{UCB}}(S\cup\{i\})\ge\mathrm{HV}(S\cup\{i\})$ and $\mathrm{HV}_t^{\mathrm{UCB}}(S)\le\mathrm{HV}(S)+C_d\sum_{j\in S}\beta_j(t)$ give the lower bound.
\end{proof}

\subsubsection{Per-round deviation stage}

Write the algorithm's own greedy construction at round $t$ as a chain $\hat S_{t,0}=\emptyset\subset\hat S_{t,1}\subset\cdots\subset\hat S_{t,k}=S_t$, where $\hat S_{t,\ell}=\hat S_{t,\ell-1}\cup\{\hat\imath_{t,\ell}\}$ and $\hat\imath_{t,\ell}\in\arg\max_{i\in A_t\setminus\hat S_{t,\ell-1}}\Delta_t^{\mathrm{UCB}}(i\mid\hat S_{t,\ell-1})$.\footnote{We assume $G^\star\subseteq A_t$, i.e.\ that the safe-pruning step never removes an arm that the true-objective greedy benchmark would have selected; this is the intended behavior of the pruning rule but is not separately proved here.} Define the \emph{first deviation stage}, comparing the algorithm's trajectory to the benchmark $G^\star$:
\[
\ell^\dagger(t)=\min\{\ell\in[k]:\hat S_{t,\ell}\ne G^\star_\ell\},
\]
with $\ell^\dagger(t)=+\infty$ if $\hat S_{t,\ell}=G^\star_\ell$ for every $\ell\in[k]$ (i.e.\ $S_t=G^\star$ exactly). This partitions $[T]$ into
\[
\mathrm{Match}=\{t:\ell^\dagger(t)=\infty\},\qquad \mathrm{Dev}_\ell=\{t:\ell^\dagger(t)=\ell\}\ (\ell\in[k]),
\qquad [T]=\mathrm{Match}\sqcup\bigsqcup_{\ell=1}^k\mathrm{Dev}_\ell.
\]
By construction, $t\in\mathrm{Dev}_\ell$ means $\hat S_{t,\ell-1}=G^\star_{\ell-1}$ — this is built into the partition, no induction over earlier rounds is needed.

Since $\bar R_T=\sum_{t=1}^T\bar r_t$ exactly by~\eqref{eq:regret-alpha}, this partition refines the same per-round sum bounded pointwise in \eqref{eq:key-bound-ptwise}: rather than controlling every term $\bar r_t$ uniformly via Lemma~\ref{lem:lip} and a single global Cauchy--Schwarz step (as in the gap-free bound above 
), we now split
\begin{equation}
\label{eq:match-dev-decomp}
\bar R_T\;=\;\sum_{t\in\mathrm{Match}}\bar r_t\;+\;\sum_{\ell=1}^k\sum_{t\in\mathrm{Dev}_\ell}\bar r_t,
\end{equation}
and bound the two kinds of terms separately: we show below that the first sum is non-positive; the remainder of this subsection bounds the second. 
\subsubsection{Regret on matched rounds is non-positive}
\label{regret-on-matched-rounds}

We bound the first term of \eqref{eq:match-dev-decomp}. If $t\in\mathrm{Match}$, then $S_t=G^\star$. By Fact~\ref{thm:greedy-generic} with $f=\mathrm{HV}$, $G=G^\star$,
\[
\bar r_t=\alpha\,V^\star-\mathrm{HV}(S_t)=\alpha\,V^\star-\mathrm{HV}(G^\star)\;\le\;\alpha\,V^\star-\alpha\,V^\star=0.
\]

\subsubsection{Witness-counting bound on deviation rounds}

Fix $\ell$ and $t\in\mathrm{Dev}_\ell$, and write $i:=\hat\imath_{t,\ell}\ne i_\ell^\star$. Since $i$ was chosen over $i_\ell^\star$ at stage $\ell$ of round $t$, $\Delta_t^{\mathrm{UCB}}(i\mid G^\star_{\ell-1})\ge\Delta_t^{\mathrm{UCB}}(i_\ell^\star\mid G^\star_{\ell-1})$. Applying Lemma~\ref{lem:marginal-dev-seq} at $S=G^\star_{\ell-1}$,
\[
\Delta(i\mid G^\star_{\ell-1})+C_d\sum_{j\in G^\star_{\ell-1}\cup\{i\}}\beta_j(t)\;\ge\;\Delta(i_\ell^\star\mid G^\star_{\ell-1})-C_d\sum_{j\in G^\star_{\ell-1}}\beta_j(t),
\]
so, using $|G^\star_{\ell-1}|=\ell-1$,
\[
\Delta_\ell(i)\;\le\; C_d\Bigl(2\sum_{j\in G^\star_{\ell-1}}\beta_j(t)+\beta_i(t)\Bigr)\;\le\; C_d(2\ell-1)\max_{j\in G^\star_{\ell-1}\cup\{i\}}\beta_j(t).
\]
Hence there exists a \emph{witness} $w(t)\in G^\star_{\ell-1}\cup\{i\}$ (at most $\ell$ candidates) with
\[
\beta_{w(t)}(t)\;\ge\;\frac{\Delta_\ell(i)}{C_d(2\ell-1)}\;\ge\;\frac{\Delta_{\min}}{C_d(2\ell-1)},
\]
\[
\text{so}\qquad
N_{w(t)}(t-1)\;\le\; m_\ell:=\frac{2\eta C_d^2(2\ell-1)^2\log(ndT^2/\delta)}{\Delta_{\min}^2}=O\!\left(\frac{\ell^2\log T}{\Delta_{\min}^2}\right).
\]
Since $G^\star_{\ell-1}\cup\{i\}=\hat S_{t,\ell}\subseteq S_t$, the witness $w(t)$ is pulled at round $t$. Group $\mathrm{Dev}_\ell$ by witness identity: for fixed $w\in[n]$, every $t\in\mathrm{Dev}_\ell$ with $w(t)=w$ increments $N_w$ by exactly $1$ while $N_w(t-1)\le m_\ell$; there are at most $m_\ell+1$ such rounds. Summing over $w\in[n]$,
\begin{equation}
\label{eq:dev-count}
|\mathrm{Dev}_\ell|\;\le\; n\,(m_\ell+1)\;=\;O\!\left(\frac{n\ell^2\log T}{\Delta_{\min}^2}\right).
\end{equation}

\subsubsection{Bounding total regret on deviation rounds}

We now bound the second term of \eqref{eq:match-dev-decomp}, $\sum_{\ell=1}^k\sum_{t\in\mathrm{Dev}_\ell}\bar r_t$. A naive approach would bound $\bar r_t\le1$ on every round of $\mathrm{Dev}_\ell$ and multiply by the count $|\mathrm{Dev}_\ell|=O(n\ell^2\log T/\Delta_{\min}^2)$ from \eqref{eq:dev-count}, giving a regret scaling as $1/\Delta_{\min}^2$. We avoid this by bounding the \emph{total} regret over $\mathrm{Dev}_\ell$ directly, via the same Cauchy--Schwarz idea as 
in the \emph{gap-free bound} Section above — applied \emph{locally} to $\mathrm{Dev}_\ell$ rather than to the full horizon. Since that argument scales as the \emph{square root} of the number of rounds involved, it converts the $1/\Delta_{\min}^2$ scale of the count into a $1/\Delta_{\min}$ contribution to regret.

On $\mathcal E$, the pointwise bound \eqref{eq:key-bound-ptwise} together with Lemma~\ref{lem:lip} gives, for every round $t$,
\[
\bar r_t\;\le\;\mathrm{HV}_t^{\mathrm{UCB}}(S_t)-\mathrm{HV}(S_t)\;\le\;C_d\sum_{i\in S_t}\beta_i(t).
\]

Fix $\ell\in[k]$, and for $i\in[n]$ let $D_i=\{t\in\mathrm{Dev}_\ell:i\in S_t\}$  be the (possibly empty) set of rounds, among those at which the trajectory first departs from $G^\star$ at stage $\ell$, where arm $i$ happens to be pulled. The pull count $N_i(t-1)$ increases by exactly $1$ at each pull of $i$, so as $t$ ranges over $D_i$ (a set of pulls of $i$), the values $N_i(t-1)$ are pairwise distinct; hence their $j$-th smallest value is at least $j-1$. Since $\beta_i(t)\propto N_i(t-1)^{-1/2}$ is decreasing, this bounds the sum by the worst case where $D_i$ consists of arm $i$'s \emph{earliest} $|D_i|$ pulls:
\[
\sum_{t\in D_i}\beta_i(t)
\;\le\;\sum_{s=1}^{|D_i|}\sqrt{\frac{2\eta\log(ndT^2/\delta)}{s}}
\;\le\;2\sqrt{2\eta\log(ndT^2/\delta)}\,\sqrt{|D_i|}
\;=\;O\bigl(\sqrt{|D_i|\log T}\bigr),
\]
using $\sum_{s=1}^m s^{-1/2}\le 2\sqrt m$ and $\log(ndT^2/\delta)=O(\log T)$ — the same argument as in 
the \emph{gap-free bound} Section, restricted to $D_i\subseteq\{$pulls of $i\}$ instead of all of $i$'s pulls over $[T]$.

Hence
\[
\sum_{t\in\mathrm{Dev}_\ell}\sum_{i\in S_t}\beta_i(t)
=\sum_{i=1}^n\sum_{t\in D_i}\beta_i(t)
=O\Bigl(\sqrt{\log T}\sum_{i=1}^n\sqrt{|D_i|}\Bigr).
\]

Since $|S_t|=k$ for every $t$, 
we have $\sum_{i=1}^n|D_i|=\sum_{t\in\mathrm{Dev}_\ell}\sum_{i=1}^n\mathbf 1\{i\in S_t\}=\sum_{t\in\mathrm{Dev}_\ell}|S_t|=k|\mathrm{Dev}_\ell|$. 
Next, by Cauchy--Schwarz (with $a_i=1$, $b_i=\sqrt{|D_i|}$),
\[
\sum_{i=1}^n\sqrt{|D_i|}\;\le\;\sqrt{n}\,\sqrt{\sum_{i=1}^n|D_i|}\;=\;\sqrt{nk|\mathrm{Dev}_\ell|}.
\]

Therefore
\[
\sum_{t\in\mathrm{Dev}_\ell}\bar r_t
\;\le\;C_d\sum_{t\in\mathrm{Dev}_\ell}\sum_{i\in S_t}\beta_i(t)
\;=\;O\bigl(\sqrt{nk|\mathrm{Dev}_\ell|\log T}\bigr)
\;=\;O\!\left(\frac{n\ell\sqrt{k}\,\log T}{\Delta_{\min}}\right),
\]
using \eqref{eq:dev-count} in the last step. Summing over $\ell\in[k]$ (using $\sum_{\ell=1}^k\ell=O(k^2)$) and combining with \eqref{eq:match-dev-decomp} and the non-positive first term,
\[
\bar R_T\;\le\;\sum_{\ell=1}^k O\!\left(\frac{n\ell\sqrt{k}\log T}{\Delta_{\min}}\right)
\;=\;O\!\left(\frac{nk^{2.5}\log T}{\Delta_{\min}}\right)\qquad\text{on }\mathcal E.
\]

\begin{theorem}[Gap-dependent bound (detailed version)]
\label{thm:gap-dependent}
On $\mathcal E$, $\bar R_T=O\bigl(nk^{2.5}\log T/\Delta_{\min}\bigr)$.
\end{theorem}

\subsection{Combining both bounds}
\label{app:combining}

Since rewards lie in $[0,1]^d$, $\bar r_t\in[0,1]$ and $\bar R_T\le T$ always (on or off $\mathcal E$). Hence
\[
\mathbb E[\bar R_T]=\mathbb E[\bar R_T\mathbf 1_{\mathcal E}]+\mathbb E[\bar R_T\mathbf 1_{\mathcal E^c}]
\le \min\Bigl\{O\bigl(C_d\sqrt{nkT\log T}\bigr),\,O\bigl(nk^{2.5}\log T/\Delta_{\min}\bigr)\Bigr\}+T\,\mathbb P(\mathcal E^c).
\]
Choosing $\delta=1/T$ gives $T\,\mathbb P(\mathcal E^c)\le1$, so:

\begin{theorem}[Regret of \textit{THV-UCB} (detailed version)]
\label{thm:main}
With $\delta=1/T$,
\[
\mathbb E[\bar R_T]\;\le\; \min\Bigl\{O\bigl(\sqrt{nkT\log T}\bigr),\;O\bigl(nk^{2.5}\log T/\Delta_{\min}\bigr)\Bigr\}+O(1).
\]
In particular $\mathbb E[\bar R_T]/T\to0$: \textit{THV-UCB} achieves sublinear $\alpha$-approximation regret.
\end{theorem}

\subsection{Discussion: what does each bound use from the algorithm?}
\label{app:discussion}

The two guarantees of Theorem~\ref{thm:main} are not in competition; they cover complementary regimes, exactly as the minimax and gap-dependent bounds do for classical $K$-armed UCB. The gap-free bound (Theorem~\ref{thm:gap-free}) is \emph{algorithm-agnostic}: its proof only uses that $k$ arms are pulled per round and that confidence widths shrink at rate $1/\sqrt{N_i(t-1)}$; it never degrades, even as $\Delta_{\min}\to0$, the regime where the gap-dependent bound diverges. Conversely, the gap-dependent bound (Theorem~\ref{thm:gap-dependent}) is \emph{mechanism-aware}: it tracks, round by round, the exact comparisons $\Delta_t^{\mathrm{UCB}}(\cdot\mid G^\star_{\ell-1})$ that the greedy step of \textit{THV-UCB} performs against the benchmark trajectory $G^\star$ — note that $G^\star$, unlike $S^\star$, never appears in the regret's definition; it is purely an internal proof device for tracking the algorithm's stage-wise progress. Its proof combines two ingredients: a witness-counting argument bounding how many rounds can deviate from $G^\star$ at each stage (see~\eqref{eq:dev-count}), and a Cauchy--Schwarz argument bounding the \emph{total} regret accrued on those rounds. 




\newpage

\section{Experimental details}\label{app:results}

\subsection{Optimal values (Grid search)}

\vspace{-0.2cm}

\begin{table*}[h!]

\centering

\small
\setlength{\tabcolsep}{3.5pt}
\resizebox{\textwidth}{!}{
\begin{tabular}{l cccc cccc cccc cccc}
\toprule
& \multicolumn{4}{c}{$d=2$} & \multicolumn{4}{c}{$d=3$} & \multicolumn{4}{c}{$d=4$} & \multicolumn{4}{c}{$d=5$} \\
\cmidrule(lr){2-5}\cmidrule(lr){6-9}\cmidrule(lr){10-13}\cmidrule(lr){14-17}
Algorithm & \rotatebox{70}{Clust.} & \rotatebox{70}{Conc.} & \rotatebox{70}{Conv.} & \rotatebox{70}{Lin.}
          & \rotatebox{70}{Clust.} & \rotatebox{70}{Conc.} & \rotatebox{70}{Conv.} & \rotatebox{70}{Lin.}
          & \rotatebox{70}{Clust.} & \rotatebox{70}{Conc.} & \rotatebox{70}{Conv.} & \rotatebox{70}{Lin.}
          & \rotatebox{70}{Clust.} & \rotatebox{70}{Conc.} & \rotatebox{70}{Conv.} & \rotatebox{70}{Lin.} \\
\midrule
THV-UCB (ours)   &0.01 &0.01 &0.01 &0.01  &0.01 &0.01 &0.01 &0.01  &0.01 &0.01 &0.01 &0.01  &0.01 &0.01 &0.01 &0.01 \\
ParetoUCB$^+$    & 1.0 & 1.0 & 1.0 & 1.0  &0.01 &0.1  &0.1  &0.3   &0.01 &0.1  &0.3  &0.3   &0.01 &0.01 &0.01 &0.1  \\
ParetoUCB-Div    &0.3  &0.1  & 1.0 &0.01  & 1.0 &0.3  &0.3  &0.3   &0.01 &0.01 &0.01 &0.3   &0.01 &0.01 & 1.0 &0.1  \\
ParetoUCB-Crowd  & 1.0 & 1.0 & 1.0 &0.3   & 1.0 &0.01 & 1.0 & 1.0  & 1.0 & 1.0 & 1.0 & 1.0  & 1.0 & 1.0 & 1.0 & 1.0 \\

ChebyshevUCB     &0.01 &0.01 & 1.0 &0.01  &0.01 &0.01 & 1.0 &0.01  &0.01 &0.01 &0.01 &0.01  &0.01 &0.01 &0.01 &0.01 \\
ChebyshevUCB$^+$ &0.3  & 1.0 & 1.0 &0.3   &0.1  &0.3  & 1.0 &0.01  &0.01 &0.01 &0.3  &0.01  &0.01 &0.01 &0.1  &0.3  \\
HVScalarUCB      &0.01 &0.01 & 1.0 &0.01  &0.01 &0.01 &0.01 &0.01  &0.01 &0.01 &0.01 &0.3   &0.01 &0.01 &0.01 &0.01 \\
HVScalarUCB$^+$  &0.01 &0.01 &0.3  &0.01  &0.01 &0.01 &0.01 &0.01  &0.1  &0.01 &0.01 &0.01  &0.01 &0.01 &0.01 &0.01 \\
ScalarUCB        &0.3  & 1.0 & 1.0 & 1.0  &0.01 &0.1  &0.1  &0.1   &0.01 &0.1  &0.3  &0.1   &0.01 &0.01 &0.01 &0.01 \\
ScalarUCB-RandW  & 1.0 & 1.0 & 1.0 & 1.0  & 1.0 &0.01 & 1.0 & 1.0  &0.01 &0.01 &0.01 & 1.0  &0.01 &0.01 &0.01 &0.01 \\

\bottomrule
\end{tabular}}\caption{Optimal confidence parameter $\eta^\star$ per algorithm and front geometry, selected by grid search over $\{0.01, 0.1, 0.3, 1.0\}$ (maximizing $\mathrm{HV}_{\mathrm{last100}}$). All detailed results are available in \url{https://ngutowski.fr/gridsearch/gridsearch.html}}
\label{tab:eta-optimal}
\end{table*}

\subsection{Additional results for $d=3$ and $k=4$}

\vspace{-0.2cm}

\begin{table*}[h!]
\centering
\small
\setlength{\tabcolsep}{4pt}
\resizebox{\textwidth}{!}{
\begin{tabular}{lccccccccc}
\toprule
& & \multicolumn{2}{c}{\textbf{Clusters}} & \multicolumn{2}{c}{\textbf{Concave}} & \multicolumn{2}{c}{\textbf{Convex}} & \multicolumn{2}{c}{\textbf{Linear}} \\
\cmidrule(lr){3-4}\cmidrule(lr){5-6}\cmidrule(lr){7-8}\cmidrule(lr){9-10}
\textbf{Algorithm} & \textbf{Fid.}
& HV$_{\text{last100}}$ & Final Regret
& HV$_{\text{last100}}$ & Final Regret
& HV$_{\text{last100}}$ & Final Regret
& HV$_{\text{last100}}$ & Final Regret \\
\midrule
THV-UCB (ours)
  & ---
  & \textbf{0.2955}{\scriptsize$\pm$0.001} & \textbf{2.8}{\scriptsize$<$0.001}
  & \textbf{0.2894}{\scriptsize$\pm$0.002} & \textbf{2.2}{\scriptsize$<$0.001}
  & \textbf{0.5895}{\scriptsize$\pm$0.003} & \textbf{2.5}{\scriptsize$<$0.001}
  & \textbf{0.1592}{\scriptsize$\pm$0.001} & 1.9{\scriptsize$\pm$0.1} \\
\addlinespace[2pt]
ParetoUCB~\citep{Drugan2013MOBandits}
  & \faithful
  & 0.1577{\scriptsize$\pm$0.034} & 141.1{\scriptsize$\pm$80.8}
  & 0.1757{\scriptsize$\pm$0.034} & 104.9{\scriptsize$\pm$78.5}
  & 0.4648{\scriptsize$\pm$0.043} & 45.4{\scriptsize$\pm$55.1}
  & 0.0610{\scriptsize$\pm$0.020} & 133.9{\scriptsize$\pm$55.1} \\
ParetoUCB\textsuperscript{+}~\citep{Drugan2013MOBandits}
  & \extension
  & 0.2844{\scriptsize$\pm$0.008} & \textbf{2.8}{\scriptsize$<$0.001}
  & 0.2755{\scriptsize$\pm$0.005} & 3.2{\scriptsize$\pm$1.3}
  & 0.5476{\scriptsize$\pm$0.011} & 24.3{\scriptsize$\pm$2.0}
  & 0.0934{\scriptsize$\pm$0.018} & 59.4{\scriptsize$\pm$40.1} \\
ParetoUCB-Div~\citep{Drugan2013MOBandits}
  & \extension
  & 0.2083{\scriptsize$\pm$0.025} & 71.9{\scriptsize$\pm$52.2}
  & 0.1988{\scriptsize$\pm$0.018} & 33.2{\scriptsize$\pm$36.5}
  & 0.5161{\scriptsize$\pm$0.012} & 6.1{\scriptsize$\pm$1.3}
  & 0.0786{\scriptsize$\pm$0.014} & 93.8{\scriptsize$\pm$46.6} \\
ParetoUCB-Crowd~\citep{Drugan2013MOBandits,Deb2002NSGA2}
  & \extension
  & 0.1694{\scriptsize$\pm$0.021} & 62.1{\scriptsize$\pm$39.0}
  & 0.1290{\scriptsize$\pm$0.008} & 135.7{\scriptsize$\pm$41.8}
  & 0.4755{\scriptsize$\pm$0.038} & 4.6{\scriptsize$\pm$3.4}
  & 0.0251{\scriptsize$\pm$0.006} & 184.3{\scriptsize$\pm$28.6} \\
\addlinespace[2pt]
ParetoTS~\citep{yahyaa2015thompson}
  & \faithful
  & 0.1708{\scriptsize$\pm$0.031} & 93.5{\scriptsize$\pm$63.4}
  & 0.1883{\scriptsize$\pm$0.030} & 60.4{\scriptsize$\pm$55.2}
  & 0.4836{\scriptsize$\pm$0.038} & 15.4{\scriptsize$\pm$20.9}
  & 0.0743{\scriptsize$\pm$0.020} & 91.7{\scriptsize$\pm$50.4} \\
ParetoTS\textsuperscript{+}~\citep{yahyaa2015thompson}
  & \extension
  & 0.2690{\scriptsize$\pm$0.011} & 12.4{\scriptsize$\pm$10.4}
  & 0.2678{\scriptsize$\pm$0.008} & 7.4{\scriptsize$\pm$6.8}
  & 0.5465{\scriptsize$\pm$0.009} & 3.0{\scriptsize$\pm$0.8}
  & 0.0868{\scriptsize$\pm$0.019} & 73.4{\scriptsize$\pm$45.1} \\
\addlinespace[2pt]
ChebyshevUCB~\citep{Mandow2023ChebyshevMAB}
  & \faithful
  & 0.1296{\scriptsize$\pm$0.055} & 246.0{\scriptsize$\pm$113.3}
  & 0.1968{\scriptsize$\pm$0.026} & 50.2{\scriptsize$\pm$26.6}
  & 0.4322{\scriptsize$\pm$0.035} & 23.0{\scriptsize$\pm$30.6}
  & 0.0980{\scriptsize$\pm$0.020} & 62.2{\scriptsize$\pm$32.8} \\
ChebyshevUCB\textsuperscript{+}~\citep{Mandow2023ChebyshevMAB}
  & \extension
  & 0.2639{\scriptsize$\pm$0.008} & 3.7{\scriptsize$\pm$0.9}
  & 0.2727{\scriptsize$\pm$0.004} & 4.4{\scriptsize$\pm$1.3}
  & 0.5437{\scriptsize$\pm$0.008} & 9.0{\scriptsize$\pm$4.0}
  & 0.1401{\scriptsize$\pm$0.001} & \textbf{1.9}{\scriptsize$\pm$0.1} \\
\addlinespace[2pt]
HVScalarUCB~\citep{zhang2020random}
  & \faithful
  & 0.2046{\scriptsize$\pm$0.036} & 76.7{\scriptsize$\pm$62.1}
  & 0.2294{\scriptsize$\pm$0.027} & 33.2{\scriptsize$\pm$37.7}
  & 0.4526{\scriptsize$\pm$0.044} & 41.4{\scriptsize$\pm$55.7}
  & 0.1252{\scriptsize$\pm$0.017} & 27.9{\scriptsize$\pm$29.0} \\
HVScalarUCB\textsuperscript{+}~\citep{Zhang2023HypervolumeRegret}
  & \faithful
  & 0.2500{\scriptsize$\pm$0.018} & 9.6{\scriptsize$\pm$10.8}
  & 0.2448{\scriptsize$\pm$0.018} & 9.1{\scriptsize$\pm$10.7}
  & 0.4956{\scriptsize$\pm$0.032} & 7.7{\scriptsize$\pm$8.8}
  & 0.1357{\scriptsize$\pm$0.010} & 6.8{\scriptsize$\pm$7.3} \\
\addlinespace[2pt]
ScalarUCB~\citep{auer2002finite}
  & \extension
  & 0.2807{\scriptsize$\pm$0.010} & \textbf{2.8}{\scriptsize$<$0.001}
  & 0.2757{\scriptsize$\pm$0.005} & 3.9{\scriptsize$\pm$2.1}
  & 0.5472{\scriptsize$\pm$0.007} & \textbf{2.6}{\scriptsize$\pm$0.1}
  & 0.0891{\scriptsize$\pm$0.019} & 68.5{\scriptsize$\pm$43.7} \\
ScalarUCB-RandW~\citep{paria2020flexible}
  & \extension
  & 0.1484{\scriptsize$\pm$0.028} & 136.8{\scriptsize$\pm$59.9}
  & 0.1644{\scriptsize$\pm$0.030} & 91.6{\scriptsize$\pm$59.2}
  & 0.4483{\scriptsize$\pm$0.038} & 32.5{\scriptsize$\pm$32.9}
  & 0.0544{\scriptsize$\pm$0.017} & 132.7{\scriptsize$\pm$46.4} \\
\addlinespace[2pt]
RandomK
  & ---
  & 0.1035{\scriptsize$\pm$0.038} & 267.2{\scriptsize$\pm$104.2}
  & 0.1110{\scriptsize$\pm$0.042} & 231.3{\scriptsize$\pm$111.7}
  & 0.3151{\scriptsize$\pm$0.091} & 275.9{\scriptsize$\pm$208.8}
  & 0.0403{\scriptsize$\pm$0.018} & 183.8{\scriptsize$\pm$53.8} \\
\bottomrule
\end{tabular}
}
\caption{Summary over four synthetic fronts ($d=3$, $n=60$, $k=4$, $\sigma=0.035$, $T=3000$, 10 seeds).
\textbf{Fid.} indicates fidelity to the cited work:
\faithful~= faithful top-$k$ adaptation (core mechanism unchanged);
\extension~= our extension (modified core, tie-break, or diversity mechanism not in the original).
\textsuperscript{+} denotes our variant when both versions share a citation.
Final Regret values are cumulative $\alpha$-regret $\pm$ 95\% CI.}
\label{tab:synthetic-summary-d3}
\end{table*}

\newpage
\subsection{Additional results for $d=4$ and $k=5$}

\begin{table*}[h!]
\centering
\small
\setlength{\tabcolsep}{4pt}
\resizebox{\textwidth}{!}{
\begin{tabular}{lccccccccc}
\toprule
& & \multicolumn{2}{c}{\textbf{Clusters}} & \multicolumn{2}{c}{\textbf{Concave}} & \multicolumn{2}{c}{\textbf{Convex}} & \multicolumn{2}{c}{\textbf{Linear}} \\
\cmidrule(lr){3-4}\cmidrule(lr){5-6}\cmidrule(lr){7-8}\cmidrule(lr){9-10}
\textbf{Algorithm} & \textbf{Fid.}
& HV$_{\text{last100}}$ & Final Regret
& HV$_{\text{last100}}$ & Final Regret
& HV$_{\text{last100}}$ & Final Regret
& HV$_{\text{last100}}$ & Final Regret \\
\midrule
THV-UCB (ours)
  & ---
  & \textbf{0.0854}{\scriptsize$\pm$0.000} & \textbf{1.5}{\scriptsize$\pm$0.0}
  & \textbf{0.1051}{\scriptsize$\pm$0.000} & \textbf{1.6}{\scriptsize$\pm$0.0}
  & \textbf{0.3709}{\scriptsize$\pm$0.002} & \textbf{3.6}{\scriptsize$\pm$0.0}
  & \textbf{0.0677}{\scriptsize$\pm$0.000} & \textbf{1.4}{\scriptsize$\pm$0.1} \\
\addlinespace[2pt]
ParetoUCB~\citep{Drugan2013MOBandits}
  & \faithful
  & 0.0255{\scriptsize$\pm$0.009} & 153.4{\scriptsize$\pm$44.6}
  & 0.0390{\scriptsize$\pm$0.012} & 150.9{\scriptsize$\pm$57.5}
  & 0.2369{\scriptsize$\pm$0.035} & 150.1{\scriptsize$\pm$123.6}
  & 0.0129{\scriptsize$\pm$0.007} & 150.5{\scriptsize$\pm$32.3} \\
ParetoUCB\textsuperscript{+}~\citep{Drugan2013MOBandits}
  & \extension
  & 0.0791{\scriptsize$\pm$0.001} & 1.6{\scriptsize$\pm$0.1}
  & 0.0947{\scriptsize$\pm$0.002} & 4.9{\scriptsize$\pm$3.1}
  & 0.3384{\scriptsize$\pm$0.006} & 9.4{\scriptsize$\pm$1.4}
  & 0.0403{\scriptsize$\pm$0.006} & 44.2{\scriptsize$\pm$25.9} \\
ParetoUCB-Div~\citep{Drugan2013MOBandits}
  & \extension
  & 0.0407{\scriptsize$\pm$0.003} & 80.2{\scriptsize$\pm$20.8}
  & 0.0492{\scriptsize$\pm$0.003} & 95.4{\scriptsize$\pm$19.7}
  & 0.2564{\scriptsize$\pm$0.009} & 19.9{\scriptsize$\pm$20.6}
  & 0.0199{\scriptsize$\pm$0.007} & 123.2{\scriptsize$\pm$35.2} \\
ParetoUCB-Crowd~\citep{Drugan2013MOBandits,Deb2002NSGA2}
  & \extension
  & 0.0109{\scriptsize$\pm$0.003} & 168.5{\scriptsize$\pm$22.6}
  & 0.0157{\scriptsize$\pm$0.006} & 215.0{\scriptsize$\pm$28.5}
  & 0.2388{\scriptsize$\pm$0.019} & 55.1{\scriptsize$\pm$41.9}
  & 0.0026{\scriptsize$\pm$0.002} & 170.3{\scriptsize$\pm$11.7} \\
\addlinespace[2pt]
ParetoTS~\citep{yahyaa2015thompson}
  & \faithful
  & 0.0302{\scriptsize$\pm$0.010} & 121.3{\scriptsize$\pm$44.7}
  & 0.0468{\scriptsize$\pm$0.012} & 112.2{\scriptsize$\pm$53.0}
  & 0.2456{\scriptsize$\pm$0.033} & 98.2{\scriptsize$\pm$93.1}
  & 0.0218{\scriptsize$\pm$0.008} & 108.0{\scriptsize$\pm$35.4} \\
ParetoTS\textsuperscript{+}~\citep{yahyaa2015thompson}
  & \extension
  & 0.0749{\scriptsize$\pm$0.003} & 14.1{\scriptsize$\pm$9.6}
  & 0.0916{\scriptsize$\pm$0.003} & 10.6{\scriptsize$\pm$8.1}
  & 0.3365{\scriptsize$\pm$0.006} & 4.5{\scriptsize$\pm$1.5}
  & 0.0390{\scriptsize$\pm$0.007} & 48.7{\scriptsize$\pm$27.6} \\
\addlinespace[2pt]
ChebyshevUCB~\citep{Mandow2023ChebyshevMAB}
  & \faithful
  & 0.0211{\scriptsize$\pm$0.015} & 183.7{\scriptsize$\pm$57.1}
  & 0.0556{\scriptsize$\pm$0.008} & 73.2{\scriptsize$\pm$24.0}
  & 0.2349{\scriptsize$\pm$0.020} & 76.4{\scriptsize$\pm$27.9}
  & 0.0279{\scriptsize$\pm$0.005} & 81.5{\scriptsize$\pm$20.1} \\
ChebyshevUCB\textsuperscript{+}~\citep{Mandow2023ChebyshevMAB}
  & \extension
  & 0.0800{\scriptsize$\pm$0.000} & \textbf{1.5}{\scriptsize$<$0.001}
  & 0.0989{\scriptsize$\pm$0.000} & \textbf{1.6}{\scriptsize$<$0.001}
  & 0.3503{\scriptsize$\pm$0.004} & 5.5{\scriptsize$\pm$1.4}
  & 0.0510{\scriptsize$\pm$0.001} & 1.9{\scriptsize$\pm$0.5} \\
\addlinespace[2pt]
HVScalarUCB~\citep{zhang2020random}
  & \faithful
  & 0.0384{\scriptsize$\pm$0.011} & 115.8{\scriptsize$\pm$50.0}
  & 0.0539{\scriptsize$\pm$0.012} & 103.0{\scriptsize$\pm$50.7}
  & 0.2510{\scriptsize$\pm$0.029} & 77.9{\scriptsize$\pm$68.3}
  & 0.0290{\scriptsize$\pm$0.008} & 91.1{\scriptsize$\pm$36.9} \\
HVScalarUCB\textsuperscript{+}~\citep{Zhang2023HypervolumeRegret}
  & \faithful
  & 0.0639{\scriptsize$\pm$0.006} & 24.3{\scriptsize$\pm$19.9}
  & 0.0852{\scriptsize$\pm$0.006} & 5.4{\scriptsize$\pm$5.6}
  & 0.3041{\scriptsize$\pm$0.020} & 13.2{\scriptsize$\pm$15.7}
  & 0.0532{\scriptsize$\pm$0.005} & 9.7{\scriptsize$\pm$9.1} \\
\addlinespace[2pt]
ScalarUCB~\citep{auer2002finite}
  & \extension
  & 0.0791{\scriptsize$\pm$0.000} & 1.7{\scriptsize$\pm$0.2}
  & 0.0946{\scriptsize$\pm$0.002} & 5.8{\scriptsize$\pm$3.5}
  & 0.3379{\scriptsize$\pm$0.006} & 7.7{\scriptsize$\pm$4.1}
  & 0.0386{\scriptsize$\pm$0.007} & 41.8{\scriptsize$\pm$25.4} \\
ScalarUCB-RandW~\citep{paria2020flexible}
  & \extension
  & 0.0353{\scriptsize$\pm$0.011} & 100.7{\scriptsize$\pm$47.3}
  & 0.0444{\scriptsize$\pm$0.014} & 118.9{\scriptsize$\pm$58.9}
  & 0.2597{\scriptsize$\pm$0.030} & 67.9{\scriptsize$\pm$73.1}
  & 0.0206{\scriptsize$\pm$0.007} & 116.6{\scriptsize$\pm$31.9} \\
\addlinespace[2pt]
RandomK
  & ---
  & 0.0168{\scriptsize$\pm$0.008} & 186.6{\scriptsize$\pm$41.9}
  & 0.0255{\scriptsize$\pm$0.012} & 203.0{\scriptsize$\pm$58.7}
  & 0.1529{\scriptsize$\pm$0.048} & 427.7{\scriptsize$\pm$222.2}
  & 0.0098{\scriptsize$\pm$0.006} & 164.4{\scriptsize$\pm$29.6} \\
\bottomrule
\end{tabular}
}
\caption{Summary over four synthetic fronts ($d=4$, $n=100$, $k=5$, $\sigma=0.025$, $T=5000$, 10 seeds).
\textbf{Fid.} indicates fidelity to the cited work:
\faithful~= faithful top-$k$ adaptation (core mechanism unchanged);
\extension~= our extension (modified core, tie-break, or diversity mechanism not in the original).
\textsuperscript{+} denotes our variant when both versions share a citation.
Final Regret values are cumulative $\alpha$-regret $\pm$ 95\% CI.}
\label{tab:synthetic-summary-d4}
\end{table*}

\subsection{Additional results for $d=5$ and $k=6$}

\begin{table*}[h!]
\centering
\small
\setlength{\tabcolsep}{4pt}
\resizebox{\textwidth}{!}{
\begin{tabular}{lccccccccc}
\toprule
& & \multicolumn{2}{c}{\textbf{Clusters}} & \multicolumn{2}{c}{\textbf{Concave}} & \multicolumn{2}{c}{\textbf{Convex}} & \multicolumn{2}{c}{\textbf{Linear}} \\
\cmidrule(lr){3-4}\cmidrule(lr){5-6}\cmidrule(lr){7-8}\cmidrule(lr){9-10}
\textbf{Algorithm} & \textbf{Fid.}
& HV$_{\text{last100}}$ & Final Regret
& HV$_{\text{last100}}$ & Final Regret
& HV$_{\text{last100}}$ & Final Regret
& HV$_{\text{last100}}$ & Final Regret \\
\midrule
THV-UCB (ours)
  & ---
  & \textbf{0.0266}{\scriptsize$<$0.001} & \textbf{0.7}{\scriptsize$\pm$0.0}
  & \textbf{0.0202}{\scriptsize$<$0.001} & \textbf{0.6}{\scriptsize$\pm$0.0}
  & \textbf{0.1841}{\scriptsize$\pm$0.001} & \textbf{3.7}{\scriptsize$<$0.001}
  & \textbf{0.0095}{\scriptsize$<$0.001} & \textbf{0.3}{\scriptsize$\pm$0.0} \\
\addlinespace[2pt]
ParetoUCB~\citep{Drugan2013MOBandits}
  & \faithful
  & 0.0042{\scriptsize$\pm$0.002} & 64.7{\scriptsize$\pm$10.6}
  & 0.0032{\scriptsize$\pm$0.002} & 48.6{\scriptsize$\pm$8.0}
  & 0.0774{\scriptsize$\pm$0.018} & 229.0{\scriptsize$\pm$86.2}
  & 0.0010{\scriptsize$\pm$0.001} & 26.1{\scriptsize$\pm$3.1} \\
ParetoUCB\textsuperscript{+}~\citep{Drugan2013MOBandits}
  & \extension
  & 0.0254{\scriptsize$<$0.001} & 0.8{\scriptsize$\pm$0.1}
  & 0.0197{\scriptsize$<$0.001} & 0.6{\scriptsize$\pm$0.1}
  & 0.1789{\scriptsize$<$0.001} & 3.7{\scriptsize$<$0.001}
  & 0.0051{\scriptsize$\pm$0.001} & 8.6{\scriptsize$\pm$4.6} \\
ParetoUCB-Div~\citep{Drugan2013MOBandits}
  & \extension
  & 0.0141{\scriptsize$\pm$0.001} & 18.8{\scriptsize$\pm$5.8}
  & 0.0094{\scriptsize$\pm$0.001} & 22.3{\scriptsize$\pm$4.7}
  & 0.1144{\scriptsize$\pm$0.012} & 65.3{\scriptsize$\pm$29.6}
  & 0.0020{\scriptsize$\pm$0.001} & 21.4{\scriptsize$\pm$4.0} \\
ParetoUCB-Crowd~\citep{Drugan2013MOBandits,Deb2002NSGA2}
  & \extension
  & 0.0025{\scriptsize$\pm$0.001} & 64.4{\scriptsize$\pm$5.3}
  & 0.0015{\scriptsize$\pm$0.001} & 51.4{\scriptsize$\pm$4.9}
  & 0.0779{\scriptsize$\pm$0.010} & 176.9{\scriptsize$\pm$48.1}
  & 0.0015{\scriptsize$\pm$0.001} & 24.7{\scriptsize$\pm$2.3} \\
\addlinespace[2pt]
ParetoTS~\citep{yahyaa2015thompson}
  & \faithful
  & 0.0057{\scriptsize$\pm$0.003} & 56.8{\scriptsize$\pm$12.3}
  & 0.0046{\scriptsize$\pm$0.002} & 41.1{\scriptsize$\pm$9.0}
  & 0.0853{\scriptsize$\pm$0.017} & 171.8{\scriptsize$\pm$76.6}
  & 0.0016{\scriptsize$\pm$0.001} & 23.3{\scriptsize$\pm$3.9} \\
ParetoTS\textsuperscript{+}~\citep{yahyaa2015thompson}
  & \extension
  & 0.0234{\scriptsize$\pm$0.001} & 8.1{\scriptsize$\pm$4.0}
  & 0.0173{\scriptsize$\pm$0.001} & 7.3{\scriptsize$\pm$3.7}
  & 0.1716{\scriptsize$\pm$0.004} & 7.5{\scriptsize$\pm$3.9}
  & 0.0050{\scriptsize$\pm$0.001} & 11.3{\scriptsize$\pm$5.0} \\
\addlinespace[2pt]
ChebyshevUCB~\citep{Mandow2023ChebyshevMAB}
  & \faithful
  & 0.0033{\scriptsize$\pm$0.003} & 69.9{\scriptsize$\pm$10.7}
  & 0.0092{\scriptsize$\pm$0.001} & 19.4{\scriptsize$\pm$3.2}
  & 0.0983{\scriptsize$\pm$0.008} & 111.9{\scriptsize$\pm$18.1}
  & 0.0047{\scriptsize$\pm$0.001} & 7.9{\scriptsize$\pm$2.9} \\
ChebyshevUCB\textsuperscript{+}~\citep{Mandow2023ChebyshevMAB}
  & \extension
  & 0.0208{\scriptsize$<$0.001} & \textbf{0.7}{\scriptsize$<$0.001}
  & 0.0150{\scriptsize$\pm$0.002} & 5.9{\scriptsize$\pm$3.2}
  & 0.1737{\scriptsize$\pm$0.003} & 6.3{\scriptsize$\pm$1.9}
  & 0.0016{\scriptsize$<$0.001} & 24.0{\scriptsize$\pm$2.3} \\
\addlinespace[2pt]
HVScalarUCB~\citep{zhang2020random}
  & \faithful
  & 0.0194{\scriptsize$\pm$0.003} & 9.5{\scriptsize$\pm$9.4}
  & 0.0123{\scriptsize$\pm$0.002} & 13.2{\scriptsize$\pm$9.1}
  & 0.1334{\scriptsize$\pm$0.017} & 46.1{\scriptsize$\pm$47.0}
  & 0.0029{\scriptsize$\pm$0.001} & 17.1{\scriptsize$\pm$5.5} \\
HVScalarUCB\textsuperscript{+}~\citep{Zhang2023HypervolumeRegret}
  & \faithful
  & 0.0218{\scriptsize$\pm$0.002} & 2.6{\scriptsize$\pm$2.5}
  & 0.0151{\scriptsize$\pm$0.002} & 4.3{\scriptsize$\pm$3.9}
  & 0.1434{\scriptsize$\pm$0.011} & 12.4{\scriptsize$\pm$12.5}
  & 0.0053{\scriptsize$\pm$0.001} & 8.1{\scriptsize$\pm$4.8} \\
\addlinespace[2pt]
ScalarUCB~\citep{auer2002finite}
  & \extension
  & 0.0254{\scriptsize$<$0.001} & 0.8{\scriptsize$\pm$0.1}
  & 0.0197{\scriptsize$<$0.001} & \textbf{0.6}{\scriptsize$\pm$0.1}
  & 0.1790{\scriptsize$<$0.001} & \textbf{3.7}{\scriptsize$\pm$0.0}
  & 0.0049{\scriptsize$\pm$0.001} & 9.0{\scriptsize$\pm$4.6} \\
ScalarUCB-RandW~\citep{paria2020flexible}
  & \extension
  & 0.0103{\scriptsize$\pm$0.004} & 38.9{\scriptsize$\pm$16.6}
  & 0.0077{\scriptsize$\pm$0.003} & 31.0{\scriptsize$\pm$12.3}
  & 0.1222{\scriptsize$\pm$0.020} & 73.4{\scriptsize$\pm$57.5}
  & 0.0027{\scriptsize$\pm$0.001} & 18.8{\scriptsize$\pm$5.7} \\
\addlinespace[2pt]
RandomK
  & ---
  & 0.0028{\scriptsize$\pm$0.002} & 69.9{\scriptsize$\pm$9.3}
  & 0.0026{\scriptsize$\pm$0.002} & 51.5{\scriptsize$\pm$7.4}
  & 0.0473{\scriptsize$\pm$0.019} & 351.6{\scriptsize$\pm$93.0}
  & 0.0008{\scriptsize$\pm$0.001} & 27.0{\scriptsize$\pm$2.8} \\
\bottomrule
\end{tabular}
}
\caption{Summary over four synthetic fronts ($d=5$, $n=150$, $k=6$,  $\sigma=0.02$, $T=5000$, 10 seeds).
\textbf{Fid.} indicates fidelity to the cited work:
\faithful~= faithful top-$k$ adaptation (core mechanism unchanged);
\extension~= our extension (modified core, tie-break, or diversity mechanism not in the original).
\textsuperscript{+} denotes our variant when both versions share a citation.
Final Regret values are cumulative $\alpha$-regret $\pm$ 95\% CI.}
\label{tab:synthetic-summary-d5}
\end{table*}

\begin{figure*}[h!]
\centering
    \begin{subfigure}[b]{0.48\textwidth}
        \includegraphics[width=\textwidth]{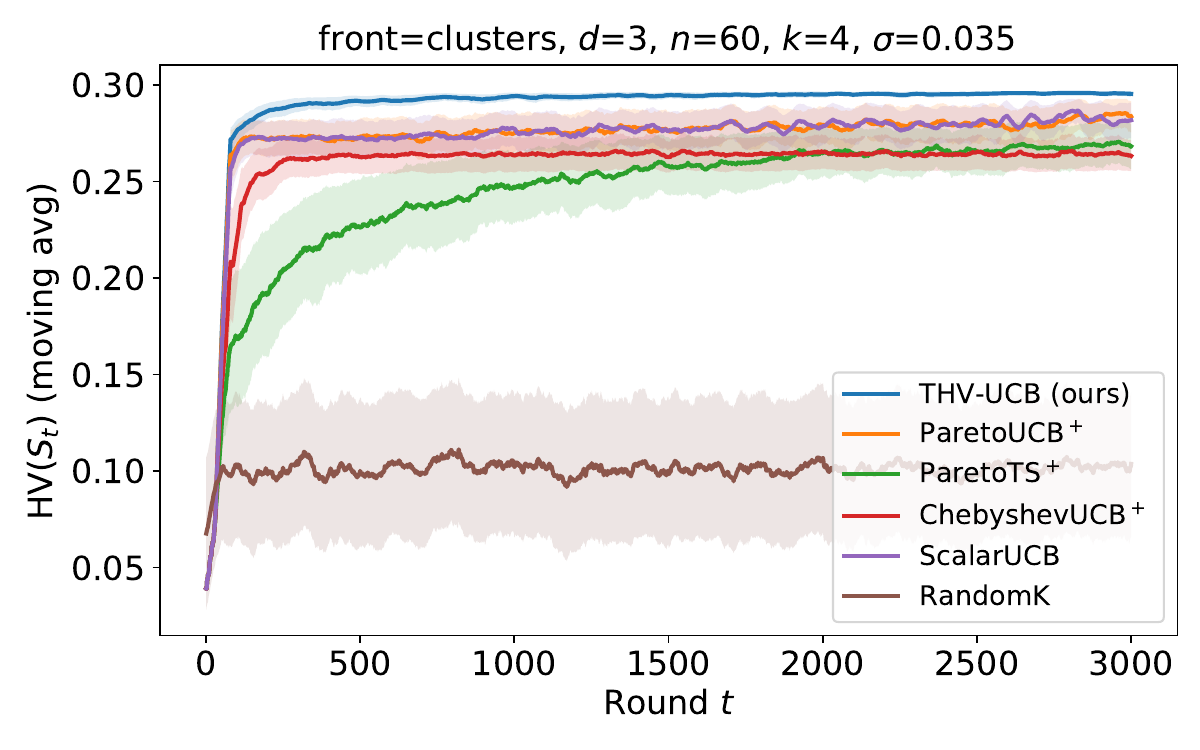}
        \caption{Clusters}
    \end{subfigure}
    \hfill
    \begin{subfigure}[b]{0.48\textwidth}
        \includegraphics[width=\textwidth]{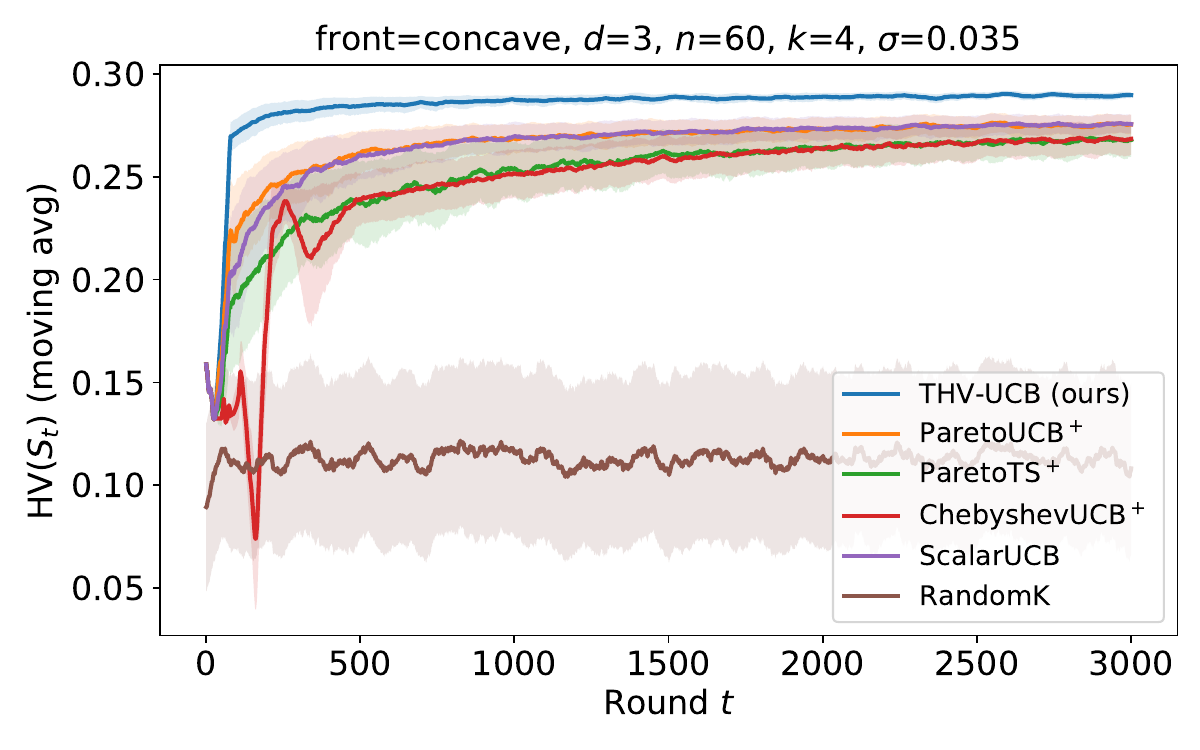}
        \caption{Concave}
    \end{subfigure}

    \vspace{0.5em}

    \begin{subfigure}[b]{0.48\textwidth}
        \includegraphics[width=\textwidth]{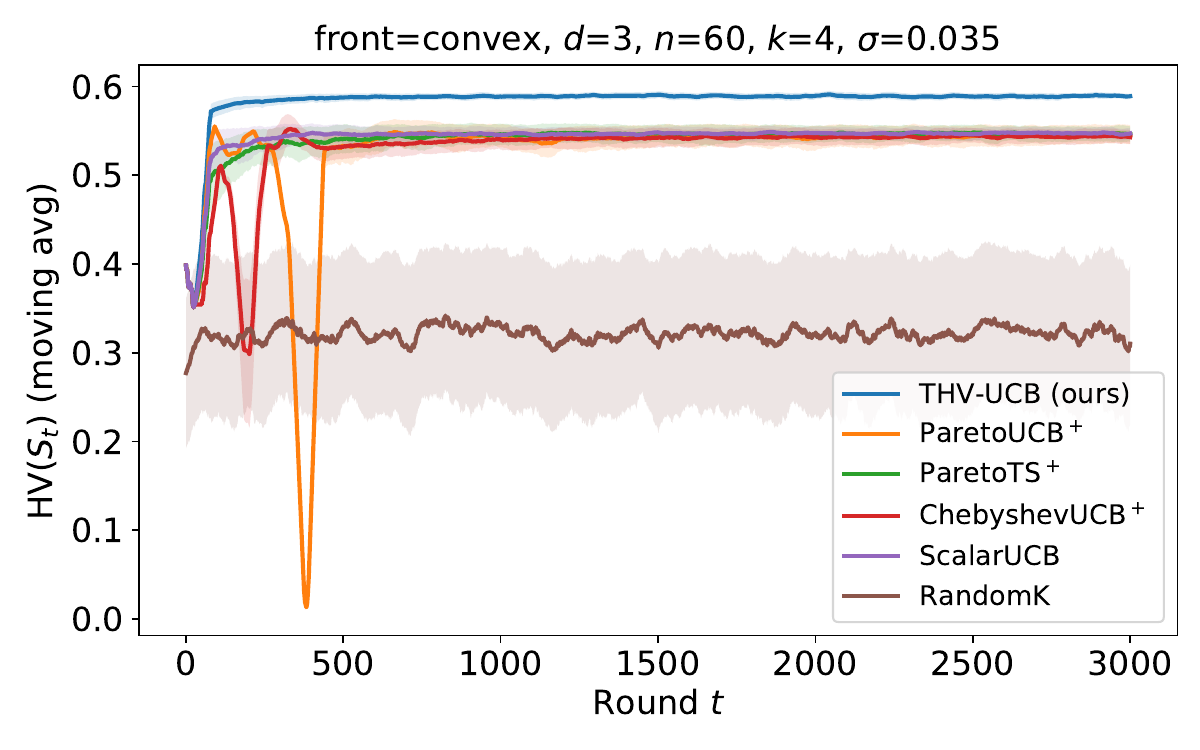}
        \caption{Convex}
    \end{subfigure}
    \hfill
    \begin{subfigure}[b]{0.48\textwidth}
        \includegraphics[width=\textwidth]{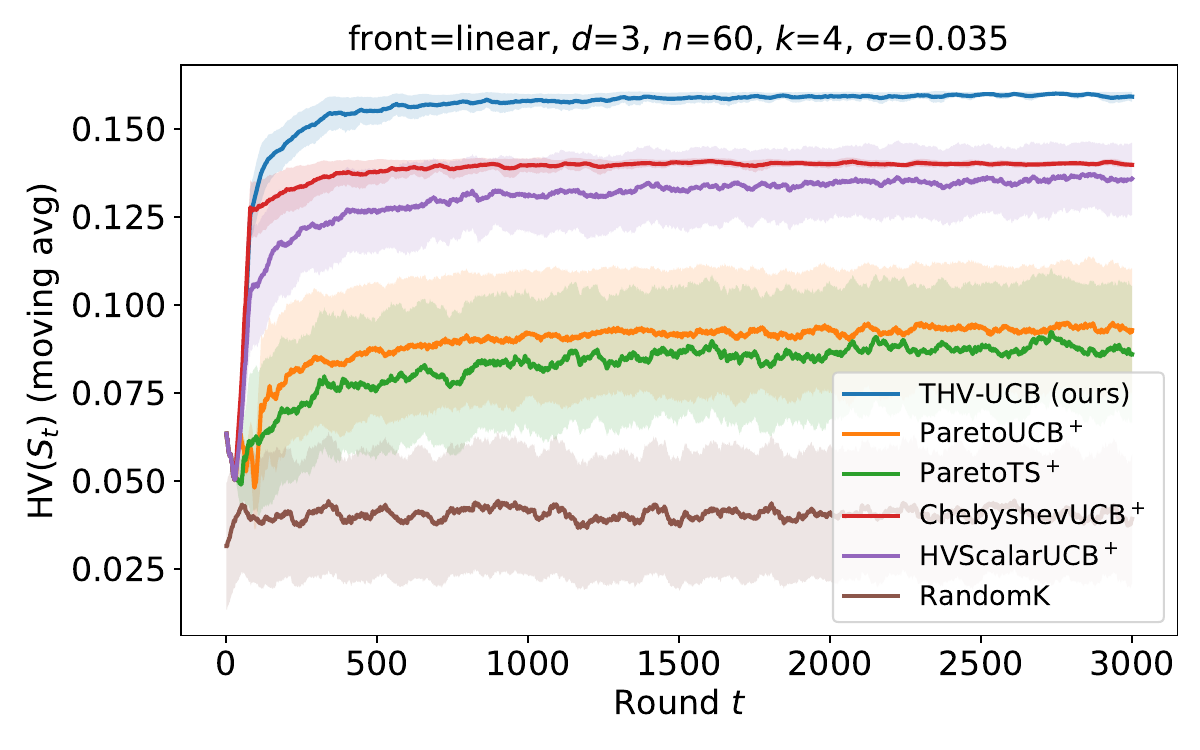}
        \caption{Linear}
    \end{subfigure}
\caption{Hypervolume $\mathrm{HV}(S_t)$ trajectories (moving average, $w=50$) with $95\%$ CIs, over four synthetic Pareto front geometries ($d=3$, $n=60$, $k=4$, $\sigma=0.035$, $T=3000$, 10 seeds). Each panel shows the best representative per algorithm family against THV-UCB (ours).}
    \label{fig:synthetic-fronts-d3}
\end{figure*}





\begin{figure*}[h!]
    \centering
    \begin{subfigure}[b]{0.48\textwidth}
        \includegraphics[width=\textwidth]{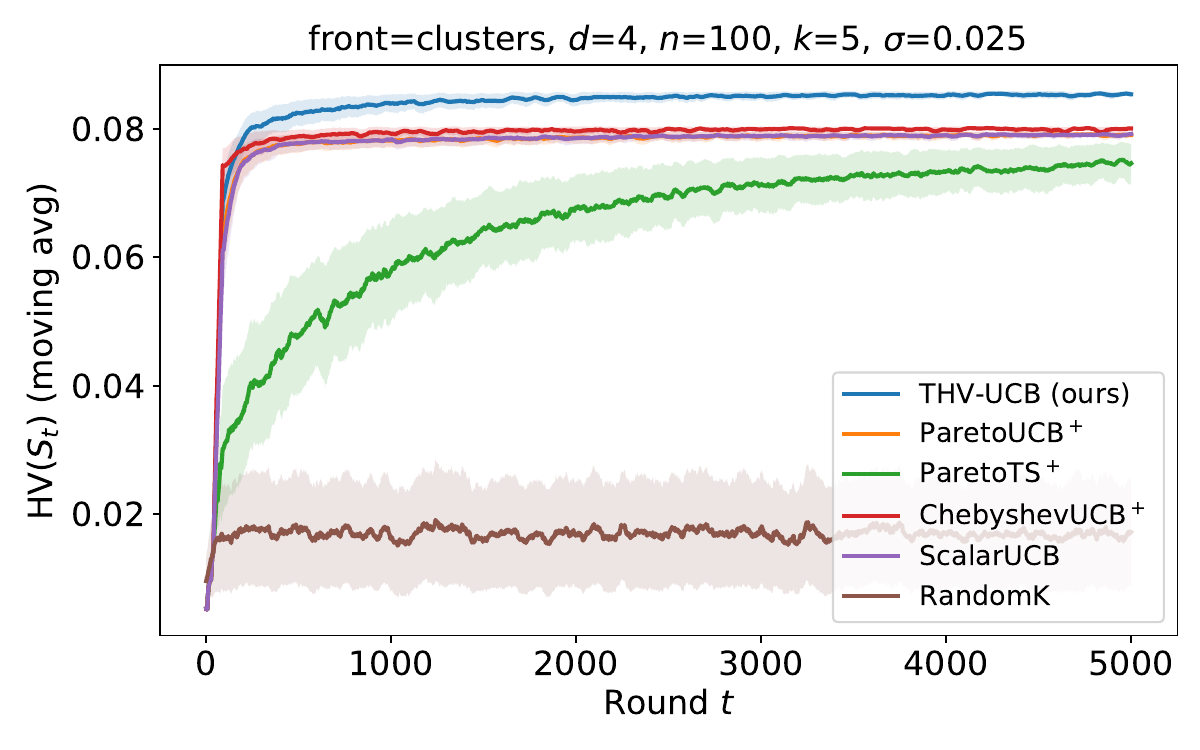}
        \caption{Clusters}
    \end{subfigure}
    \hfill
    \begin{subfigure}[b]{0.48\textwidth}
        \includegraphics[width=\textwidth]{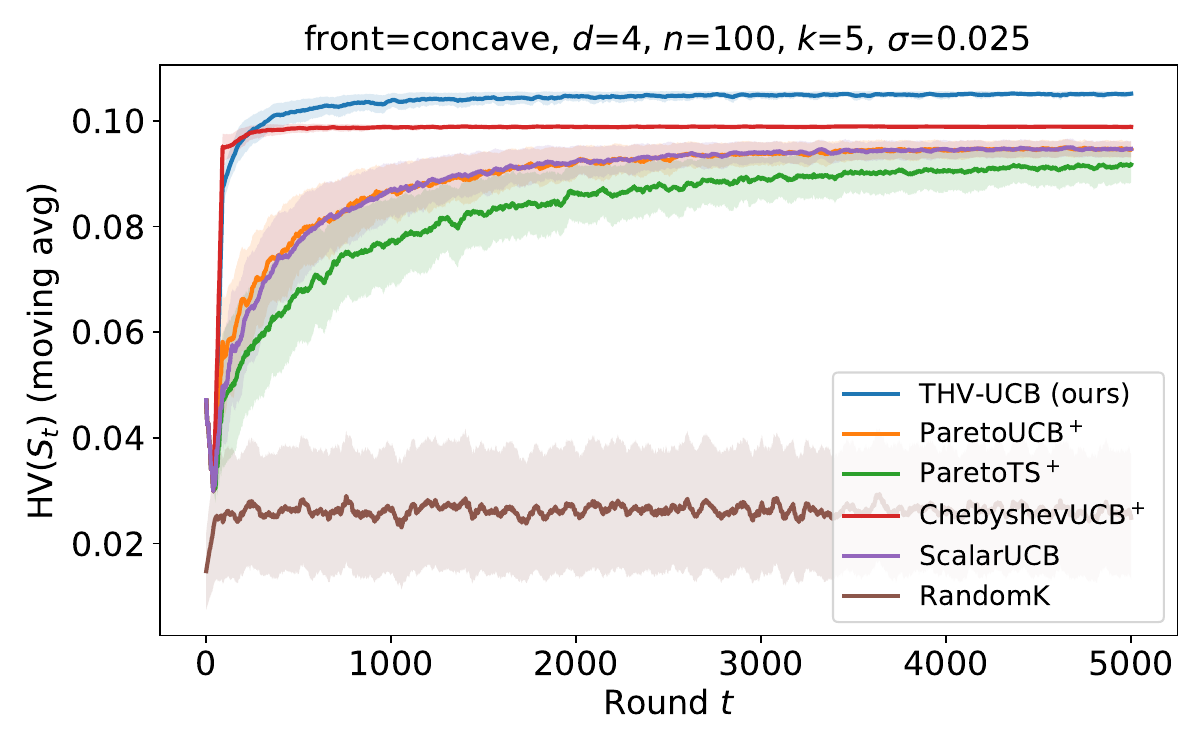}
        \caption{Concave}
    \end{subfigure}

    \vspace{0.5em}

    \begin{subfigure}[b]{0.48\textwidth}
        \includegraphics[width=\textwidth]{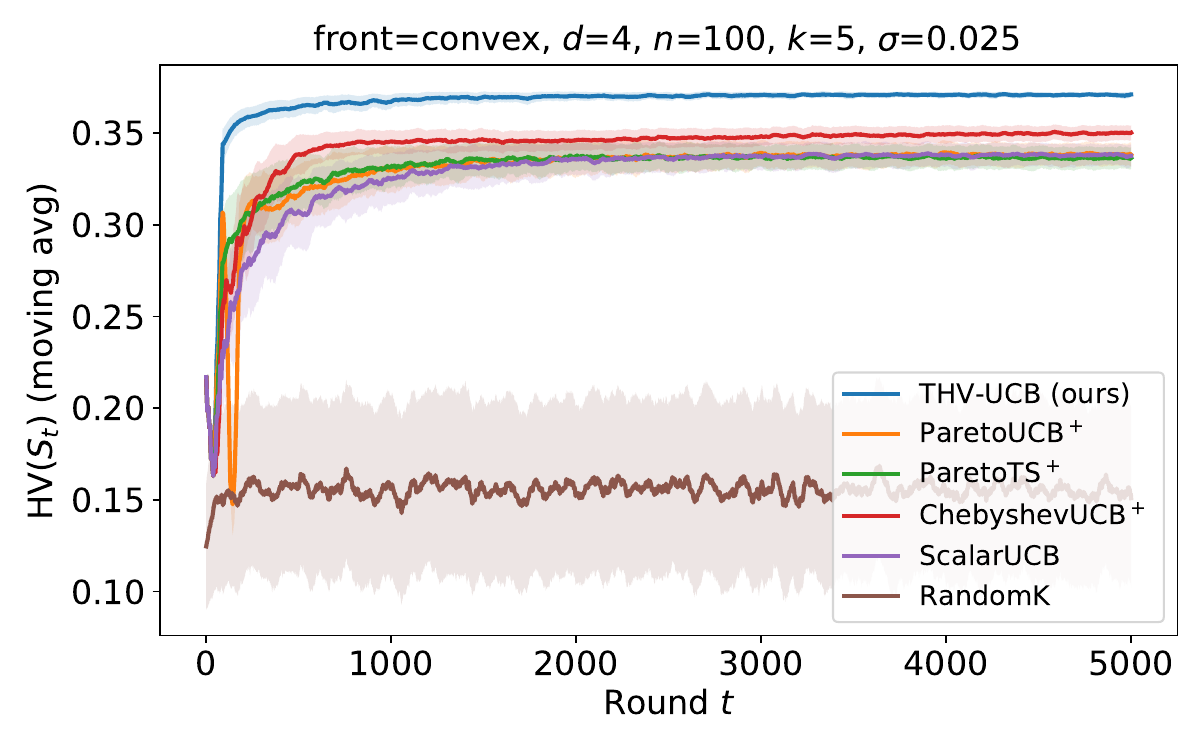}
        \caption{Convex}
    \end{subfigure}
    \hfill
    \begin{subfigure}[b]{0.48\textwidth}
        \includegraphics[width=\textwidth]{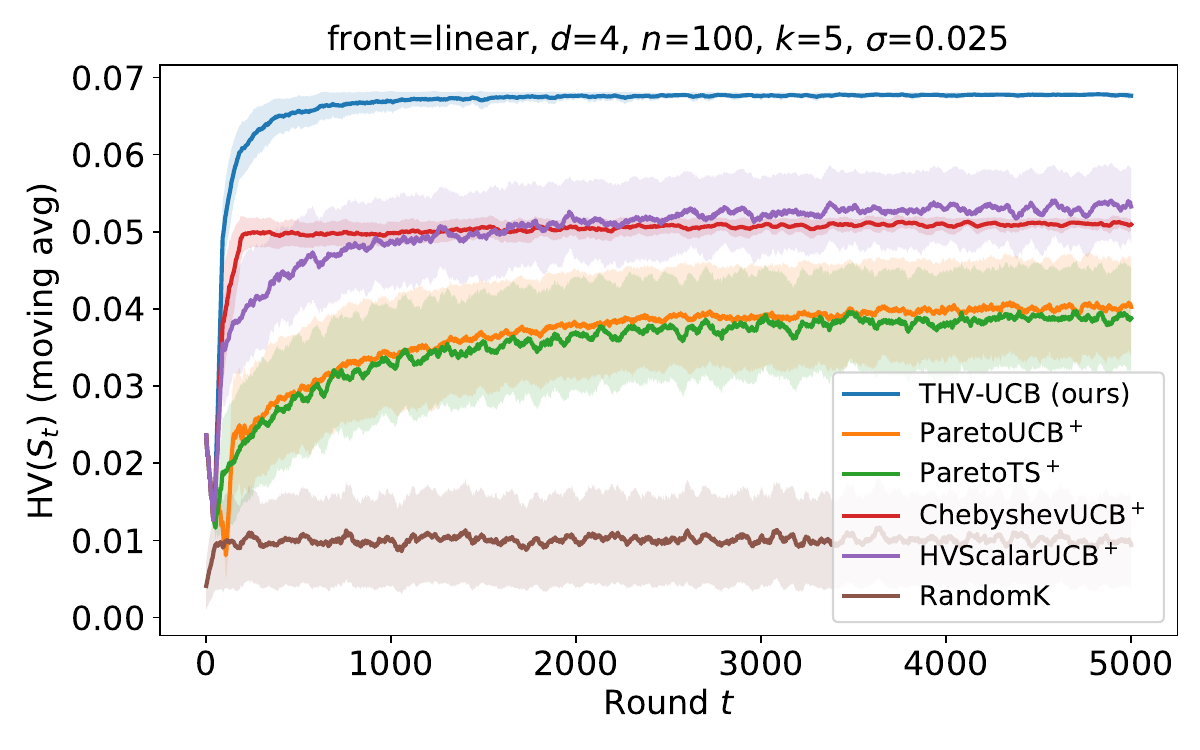}
        \caption{Linear}
    \end{subfigure}
    \caption{Hypervolume $\mathrm{HV}(S_t)$ trajectories (moving average, $w=50$) with $95\%$ CIs, over four synthetic Pareto front geometries ($d=4$, $n=100$, $k=5$, $\sigma=0.025$, $T=5000$, 10 seeds). Each panel shows the best representative per algorithm family against THV-UCB (ours).}
    \label{fig:synthetic-fronts-d4}
\end{figure*}



\begin{figure*}[h!]
    \centering
    \begin{subfigure}[b]{0.48\textwidth}
        \includegraphics[width=\textwidth]{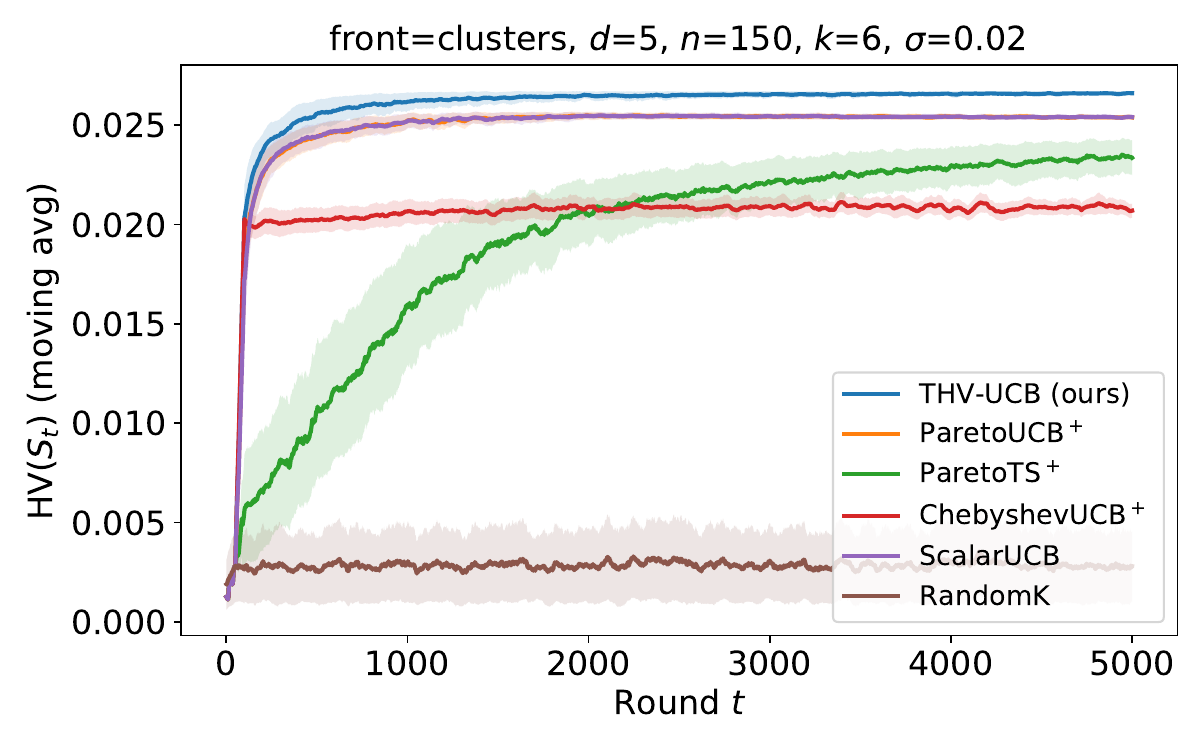}
        \caption{Clusters}
    \end{subfigure}
    \hfill
    \begin{subfigure}[b]{0.48\textwidth}
        \includegraphics[width=\textwidth]{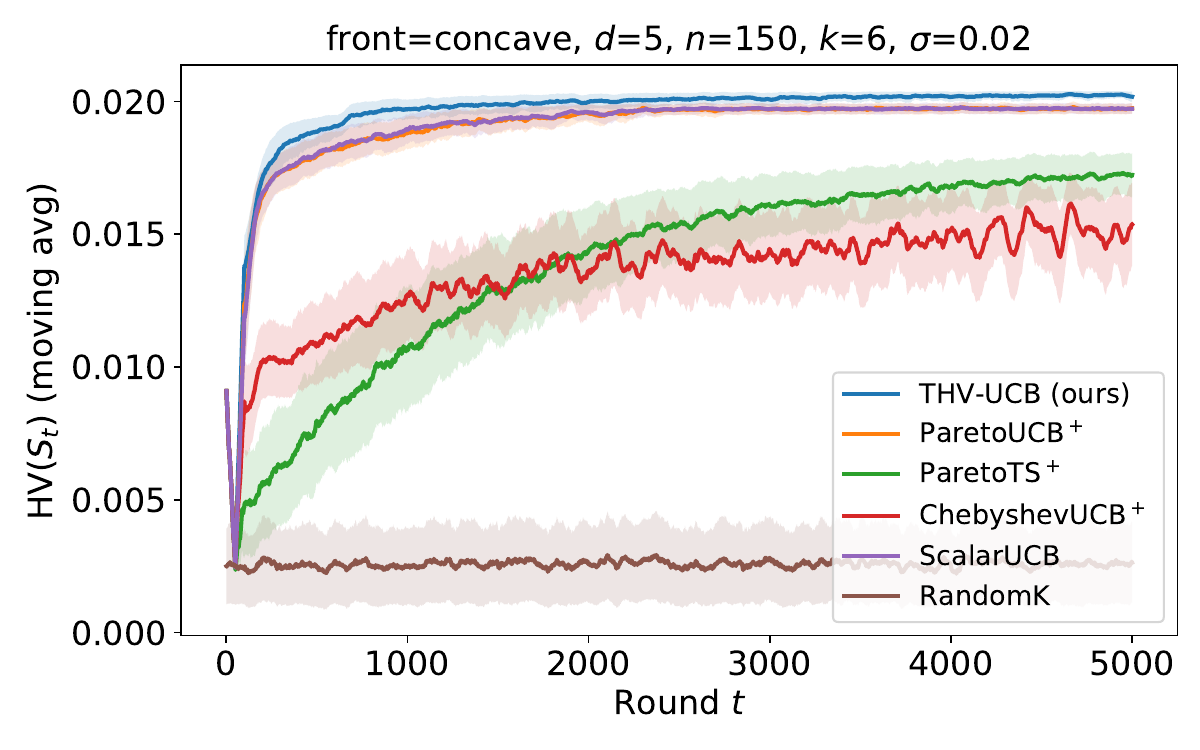}
        \caption{Concave}
    \end{subfigure}

    \vspace{0.5em}

    \begin{subfigure}[b]{0.48\textwidth}
        \includegraphics[width=\textwidth]{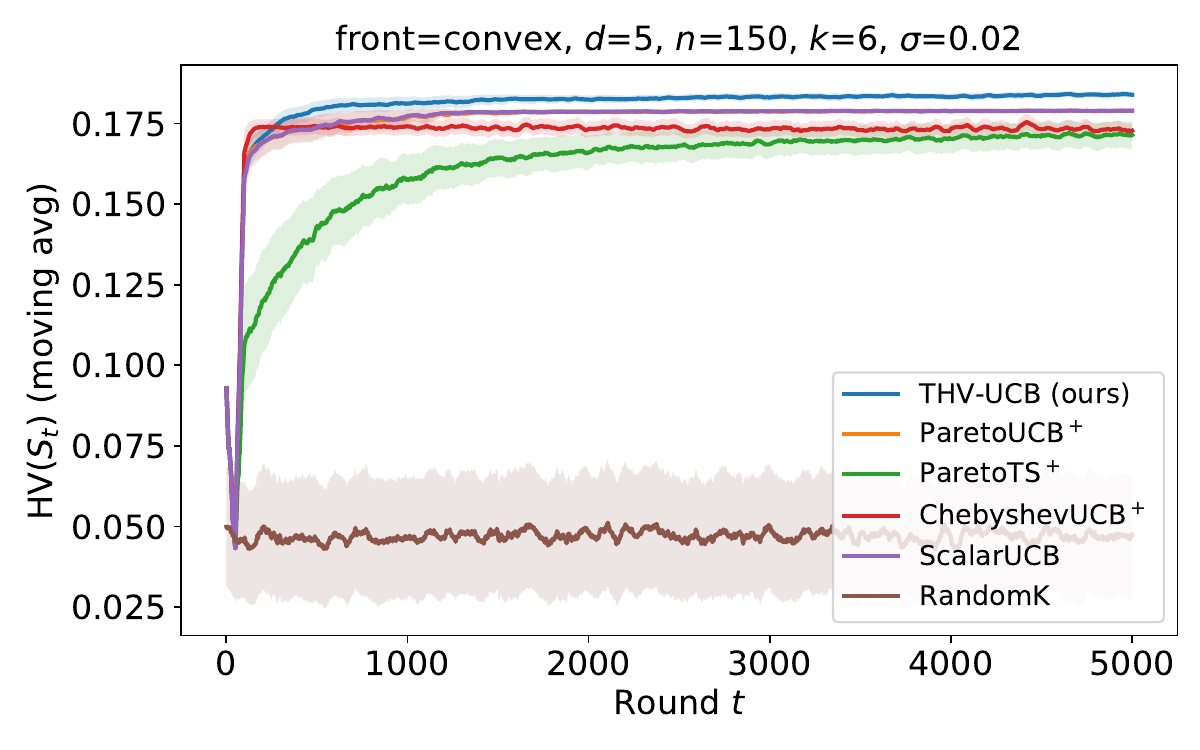}
        \caption{Convex}
    \end{subfigure}
    \hfill
    \begin{subfigure}[b]{0.48\textwidth}
        \includegraphics[width=\textwidth]{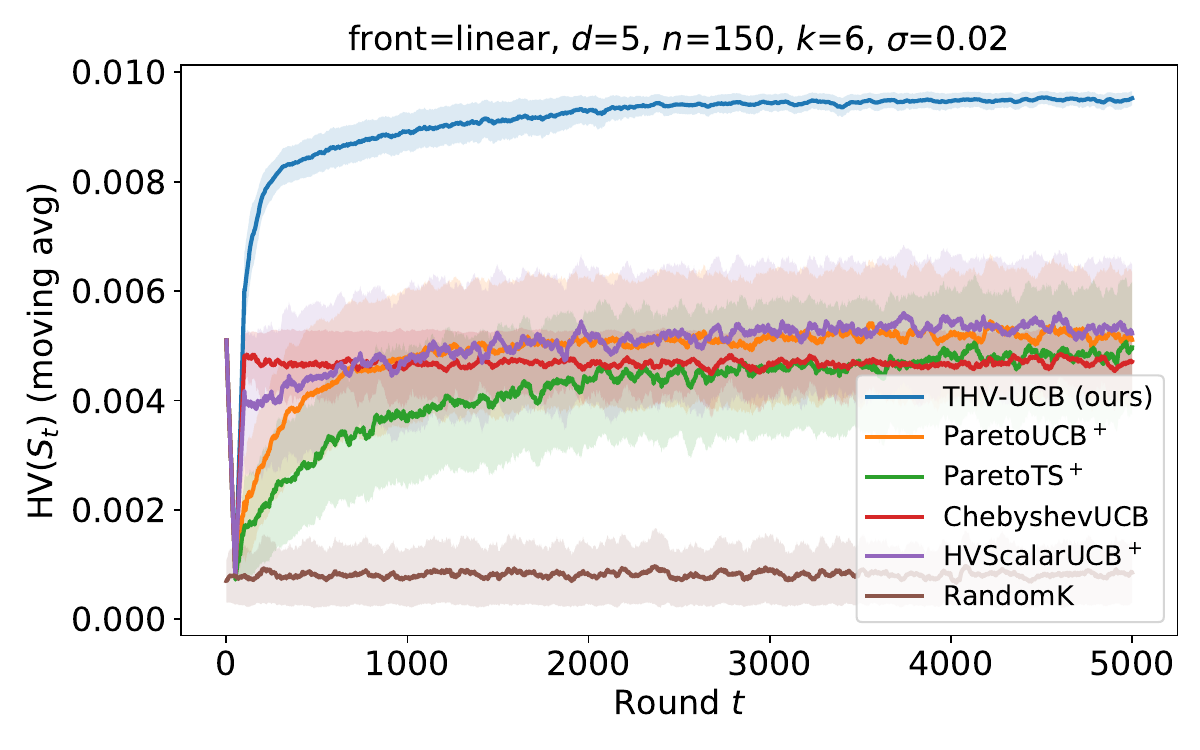}
        \caption{Linear}
    \end{subfigure}
    \caption{Hypervolume $\mathrm{HV}(S_t)$ trajectories (moving average, $w=50$) with $95\%$ CIs, over four synthetic Pareto front geometries ($d=5$, $n=150$, $k=6$,  $\sigma=0.02$, $T=5000$, 10 seeds). Each panel shows the best representative per algorithm family against THV-UCB (ours).}
    \label{fig:synthetic-fronts-d5}
\end{figure*}




\clearpage

\subsection{Statistical tests}
\begin{table*}[ht!]
\centering

\setlength{\tabcolsep}{4pt}
\resizebox{\textwidth}{!}{%
\begin{tabular}{llcccccc}
\toprule
\textbf{Setting} & \textbf{Baseline} & $\boldsymbol{\Delta}$ & \textbf{Bootstrap 95\% CI} & \textbf{Cohen's $d$} & \textbf{Wilcoxon $p$} & \textbf{Wins} \\
\midrule
$d=2$, clusters & ParetoUCB-Div    & $+0.0048$ & $[+0.0034,\, +0.0063]$ & $1.92$  & $<0.001$ & $10/10$ \\
$d=3$, concave  & ChebyshevUCB$^+$ & $+0.0169$ & $[+0.0161,\, +0.0177]$ & $12.32$ & $<0.001$ & $10/10$ \\
$d=3$, convex   & ParetoTS$^+$     & $+0.0430$ & $[+0.0418,\, +0.0442]$ & $20.45$ & $<0.001$ & $10/10$ \\
$d=4$, clusters & ChebyshevUCB$^+$ & $+0.0054$ & $[+0.0053,\, +0.0055]$ & $31.77$ & $<0.001$ & $10/10$ \\
$d=4$, concave  & ChebyshevUCB$^+$ & $+0.0062$ & $[+0.0061,\, +0.0063]$ & $44.48$ & $<0.001$ & $10/10$ \\
$d=4$, linear   & HVScalarUCB$^+$  & $+0.0145$ & $[+0.0141,\, +0.0148]$ & $22.84$ & $<0.001$ & $10/10$ \\
$d=5$, clusters & ScalarUCB        & $+0.0012$ & $[+0.0012,\, +0.0012]$ & $29.98$ & $<0.001$ & $10/10$ \\
$d=5$, concave  & ScalarUCB        & $+0.0005$ & $[+0.0005,\, +0.0005]$ & $8.09$  & $<0.001$ & $10/10$ \\
$d=5$, convex   & ScalarUCB        & $+0.0051$ & $[+0.0049,\, +0.0053]$ & $14.97$ & $<0.001$ & $10/10$ \\
$d=5$, linear   & HVScalarUCB$^+$  & $+0.0042$ & $[+0.0041,\, +0.0044]$ & $19.55$ & $<0.001$ & $10/10$ \\
\bottomrule
\end{tabular}%
}\small
\caption{Paired statistical comparisons between THV-UCB and the second-best baseline on the settings with the tightest margins (10 random seeds, paired by environment seed).
$\Delta$ denotes the mean difference in $\mathrm{HV}_{\text{last100}}$ (THV-UCB $-$ baseline).
Bootstrap 95\% CIs are computed over $B=10{,}000$ resamples.
Wilcoxon signed-rank tests are one-sided ($H_1$: THV-UCB $>$ baseline);
\textit{Wins} = number of seeds (out of 10) on which THV-UCB outperforms the baseline.}
\label{tab:stat-tests}
\end{table*}

\clearpage

\section{Notation summary}\label{app:notation}

\begin{table}[h!]
\centering
\small
\resizebox{0.75\columnwidth}{!}{
\begin{tabular}{@{}ll@{}}
\toprule
Symbol & Meaning \\
\midrule
$n$ & Number of arms \\
$T$ & Time horizon (rounds) \\
$d$ & Number of objectives (reward dimensions) \\
$k$ & Subset size selected each round \\
$[n]$ & $\{1,\dots,n\}$ \\
$[T]$ & $\{1,\dots,T\}$ \\
$S_t \subseteq [n]$ & Slate (set) of $k$ arms chosen at round $t$ \\
$\mathbf{X}_{i,t} \in [0,1]^d$ & Reward vector of arm $i$ at round $t$ \\
$\boldsymbol{\mu}_i = \mathbb{E}[\mathbf{X}_{i,t}]$ & Mean reward vector of arm $i$ \\
$\widehat{\boldsymbol{\mu}}_i(t)$ & Empirical mean of arm $i$ up to round $t$ \\
$N_i(t)$ & Number of times arm $i$ was selected up to round $t$ \\
$\mathcal{F}_t$ & Filtration generated by observations up to round $t$ \\
$\mathbf{r} \in \mathbb{R}^d$ & Reference point for hypervolume (dominated by all means) \\
$\mathrm{HV}(S)$ & Dominated hypervolume of set $S$  \\
$\mathcal{P}$ & Pareto front (set of undominated means) \\
$\preceq, \prec$ & (Strict) Pareto dominance relations \\
$\eta$ & Confidence parameter; also governs sub-Gaussianity of reward coordinates \\
$\beta_i(t)$ & Confidence radius for arm $i$ at round $t$ \\
$S^\star \in \arg\max_{|S|=k}\mathrm{HV}(S)$ & Best size-$k$ subset in hindsight \\
$V^\star = \mathrm{HV}(S^\star)$ & Benchmark hypervolume value \\
$G^\star = \{i^\star_1,\dots,i^\star_k\}$ & Greedy benchmark on true means \\
$G^\star_\ell$ & Length-$\ell$ prefix of the greedy benchmark $G^\star$ \\
$\Delta(i\mid S) = \mathrm{HV}(S\cup\{i\}) - \mathrm{HV}(S)$ & True marginal hypervolume gain \\
$\Delta_\ell(i)$ & Stage-$\ell$ gap of arm $i$ relative to $G^\star$ \\
$\Delta_{\min} = \min_{\ell\in[k]}\min_{i\neq i^\star_\ell}\Delta_\ell(i)$ & Minimum stage-wise gap \\
$C_d$ & Coordinate-wise Lipschitz constant of $\mathrm{HV}$ ($C_d \leq d$) \\
$\mathcal{E}$ & High-probability event on which all confidence intervals hold \\
$r_t = V^\star - \mathrm{HV}(S_t)$ & Instantaneous regret w.r.t.\ $V^\star$  \\
$R_T=\sum_{t=1}^T r_t$ & Cumulative regret  \\
$\alpha$ & Greedy approximation factor ($1-1/e$) \\
$\bar r_t = \alpha V^\star - \mathrm{HV}(S_t)$ & Instantaneous $\alpha$-regret \\
$\bar R_T=\sum_{t=1}^T \bar r_t$ & Cumulative $\alpha$-regret  \\
$\|\cdot\|_\infty$ & $\ell_\infty$ norm \\
\bottomrule
\end{tabular}
}
\caption{Notation summary.}
\label{tab:notation}
\end{table}